\documentclass[journal]{IEEEtran}
\usepackage{cite}

\usepackage{fancyhdr}

\usepackage{lastpage}

\ifCLASSINFOpdf
   \usepackage[pdftex]{graphicx}
\else
   \usepackage[dvips]{graphicx}
\fi
\usepackage{xcolor}
\usepackage{color}
\usepackage{amsmath}
\usepackage{amssymb}
\usepackage{mathrsfs}
\usepackage{amsthm}
\usepackage{algorithmic}
\usepackage[ruled,vlined]{algorithm2e}
\usepackage{array}
\ifCLASSOPTIONcompsoc
  \usepackage[caption=false,font=normalsize,labelfont=sf,textfont=sf]{subfig}
\else
  \usepackage[caption=false,font=footnotesize]{subfig}
\fi
\usepackage{fixltx2e}
\usepackage{dblfloatfix}

\newtheorem{assumption}{Assumption}{}
\newtheorem{lemma}{Lemma}{}
\newtheorem{theorem}{Theorem}{}
{}
{}

\begin{document}

\title{Large-scale Traffic Signal Control Using a Novel Multi-Agent Reinforcement Learning}
\author{Xiaoqiang~Wang,
        Liangjun~Ke,~\IEEEmembership{Member,~IEEE,}
        Zhimin~Qiao,
        and~Xinghua~Chai 
\thanks{X. Wang, L. Ke and Z. Qiao are with
 State Key Laboratory for Manufacturing Systems Engineering, School of Automation Science and Engineering,
Xi'an Jiaotong University, Xi'an, Shaanxi, 710049, China. (e-mail: wangxq5127@stu.xjtu.edu.cn; keljxjtu@xjtu.edu.cn; qiao.miracle@gmail.com)(Corresponding author: Liangjun~Ke.).}
\thanks{X. Wang and X. Chai are with CETC Key Laboratory of Aerospace Information Applications, Shijiazhuang, Hebei, China (e-mail: cetc54008@yeah.net).} }

\maketitle
\thispagestyle{fancy}         
\fancyhead{}
\fancyhead[C]{\scriptsize{This is the author's version of an article that has been published in IEEE Transaction on Cybernetics. Changes were made to this version by the publisher prior to publication.
The final version of record is available at http://dx.doi.org/10.1109/TCYB.2020.3015811}}
\fancyfoot[C]{\scriptsize{Copyright (c) 2020 IEEE. Personal use is permitted. For any other purposes, permission must be obtained from the IEEE by emailing pubs-permissions@ieee.org.}\\ \rfoot{\scriptsize{\text{Page} \thepage}}}

\begin{abstract}
Finding the optimal signal timing strategy is a difficult task for the problem of large-scale traffic signal control (TSC). Multi-Agent Reinforcement Learning (MARL) is a promising method to solve this problem.
However,  there is still room for improvement in extending to large-scale problems and modeling the behaviors of other agents for each individual agent.
In this paper, a new MARL, called  \emph{Cooperative double Q-learning} (Co-DQL), is proposed, which has several prominent features. It uses a highly scalable independent double Q-learning method based on double estimators and the upper confidence bound (UCB) policy, which can eliminate the over-estimation problem existing in traditional independent Q-learning while ensuring exploration. It uses mean field approximation to model the interaction among agents, thereby making agents learn a better cooperative strategy.
In order to improve the stability and  robustness of the learning process, we introduce a new reward allocation mechanism and a local state sharing method.
In addition, we analyze the convergence properties of the proposed algorithm.
Co-DQL is applied to TSC and tested on various traffic flow scenarios of TSC simulators. The results show that Co-DQL outperforms the state-of-the-art decentralized MARL algorithms in terms of multiple traffic metrics.
\end{abstract}


\IEEEpeerreviewmaketitle

\section{Introduction}

\IEEEPARstart{T}{raffic} congestion is becoming a great puzzling problem in urban, mainly due to the difficulty of effective utilization of limited road resources  (e.g. road width). By regulating   traffic flow of road network, the traffic signal control (TSC) at intersections plays an important role in utilizing the road resources and helping  to  reduce traffic congestion \cite{yau2017survey}.

Many researchers have devoted efforts to TSC, with the aim of minimizing the average waiting time in the whole traffic system  and maximizing social welfare \cite{wu2004optimal}.
When traffic signals  are large-scale, the traditional control methods such as pre-timed \cite{yin2016traffic} and actuated control systems \cite{koonce2008traffic} may fail to deal with the dynamic of the traffic conditions or lack the ability to foresee traffic flow.
Intelligent computing methods (such as genetic algorithm \cite{ceylan2004traffic}, swarm intelligence \cite{garcia2012swarm}, neuro-fuzzy networks \cite{qiao2010two} \cite{srinivasan2006neural}), however, in many cases, suffer from a slow convergence rate.
Reinforcement learning (RL) \cite{sutton2018reinforcement} is a promising adaptive decision-making method in many fields. It has been applied to cope with TSC \cite{wiering2004intelligent} \cite{prashanth2010reinforcement}. It can not only make real-time decisions according to traffic flow, but also predict future traffic flow. Especially in recent years, RL has made tremendous progress which significantly attributes to the success of deep learning \cite{lecun2015deep}.
By using deep neural network to approximate the value function or action-value function (such as DQN \cite{mnih2015human}, DDPG \cite{lillicrap2015continuous}), RL can be adapted to the problems with large-scale state space or action space.

As for TSC with multiple signalized intersections, a straightforward idea is centralized, in which TSC is considered as a single-agent learning problem \cite{ceylan2004traffic} \cite{wei2018intellilight}. However, a centralized approaches often need to collect all traffic data in the network as the global state \cite{casas2017deep}, which may lead to high latency and failure rate. In addition, as the number of intersections increases, the joint state space and action space of the agent will increase exponentially to a large extent, which incurs the curse of dimension. Consequently, a centralized method often requires very heavy computational and communication burden.

An alternative way is multi-agent reinforcement learning (MARL) in which each  signalized intersection is regarded as an agent. A challenge of a MARL approach is how to response to the dynamic interaction between each signal agent and the environment, which significantly affects the adaptive decision-making of other signals \cite{claus1998dynamics}.
Moreover, most of the current MARL methods are only studied on very limited-size traffic network problems \cite{shamshirband2012distributed} \cite{arel2010reinforcement}. However, in urban traffic systems, it is often necessary to consider all the signals in a global coordination manner.
In \cite{abdoos2011traffic} \cite{tan1993multi}, each signal is regarded as an independent agent for training.
Although this class of approaches can easily be extended to large-scale scenarios, they directly ignore the actions of other agents in the road network system and implicitly suggest that the environment is static. This makes it difficult for agents to learn favorable strategies with convergence guarantee.
In \cite{kuyer2008multiagent}, a max-plus method is proposed to deal with large-scale TSC problem, but this approach requires additional computation during execution.
Multi-agent A2C \cite{chu2019multi} is developed from IA2C which is scalable and belongs to a decentralized MARL algorithm, but it may be uneasy to determine the appropriate attenuation factor to weaken the state and reward information from other agents.

In this work, we present a decentralized and scalable MARL method which is named after  Cooperative double Q-learning (Co-DQL) and apply it to TSC.
The new approach adopts a highly scalable independent double Q-learning method, with the aim of avoiding the problem of over-estimation suffered from traditional independent Q-learning \cite{hasselt2010double}. At the mean time, it can ensure exploration by using the upper confidence bound (UCB) \cite{auer2002finite} rule.
In order to make agents learn a better cooperative strategy for large-scale problems, it employs mean field theory \cite{stanley1971phase}, which  has been studied in \cite{yang2018mean}.  It approximately  treats the interactions within the population of  agents as the interaction between a single agent and a virtual  agent averaged by other individuals, which potentially transmits the action information among all  agents in the environment.
Furthermore, we introduce a new reward allocation mechanism and a local state sharing method to make the learning process of agents more stable and robust.
To theoretically support the effectiveness of the proposed algorithm, we  provide the convergence proof for the proposed algorithm under some mild conditions.
Numerical experiment is performed on various traffic flow scenarios of TSC simulators. The empirical results show that the proposed method outperforms several state-of-the-art decentralized MARL algorithms  in terms of multiple traffic metrics.

The paper is organised into six sections. Section~\ref{sec2} describes the background on RL. Section~\ref{sec3} presents the proposed method and analyzes the convergence properties. Section~\ref{sec4} introduces the application of Co-DQL to TSC problem. Section~\ref{sec5} describes the setup and conditions of the experiments in detail, and makes a comparative analysis and discussion on the experimental results. Section~\ref{sec6} summarises  this paper.

\section{Background on Reinforcement Learning}\label{sec2}
\subsection{Single-Agent RL}\label{subsec2.1}
Q-learning is one of the most popular RL methods and it solves   sequential decision-making problems by learning estimates for the optimal value of each action. The optimal value can be expressed as $Q^{*}(s, a)=\max _{\pi} Q^{\pi}(s, a)$. However, it is not easy to learn the values of all the actions in all states when the state space or action space is larger. In this case, we can learn a parameterized action-value function $Q(s, a; \boldsymbol{\theta})$. When taking action $a_t$ in state $s_t$ and observing the immediate reward $r_{t+1}$ and resulting state $s_{t+1}$ , the standard Q-learning updates the parameters as follows:
\begin{equation}\label{eq104}
\boldsymbol{\theta}_{t+1}=\boldsymbol{\theta}_{t}+\alpha\left(Y_{t}^{\mathrm{Q}}-Q(s_{t}, a_{t} ; \boldsymbol{\theta}_{t})\right) \nabla_{\boldsymbol{\theta}_{t}} Q(s_{t}, a_{t} ; \boldsymbol{\theta}_{t}),
\end{equation}
where $t$ is the time step, $\alpha$ is the learning rate and the target $Y^Q_t$ is defined as:
\begin{equation}\label{eq105}
Y_{t}^{\mathrm{Q}} \equiv r_{t+1}+\gamma \max _{a} Q(s_{t+1}, a; \boldsymbol{\theta}_{t}),
\end{equation}
where the constant $\gamma\in[0, 1)$ is the discount factor that trades off the importance of immediate and later rewards. After updating gradually, it can converge to optimal action-value function.

Note that Q-learning approximates the value of the next state by maximizing over the estimated action values in the corresponding state, namely, $\max _{a} Q_{t}(s_{t+1}, a; \boldsymbol{\theta}_{t})$ and it is an estimate of $E\left\{\max _{a} Q_{t}\left(s_{t+1}, a; \boldsymbol{\theta}_{t}\right)\right\}$, which in turn is used to approximate $\max _{a} E\left\{Q_{t}\left(s_{t+1}, a; \boldsymbol{\theta}_{t}\right)\right\}$. This method of approximating the maximum expected value has a positive deviation \cite{hasselt2010double} \cite{van2016deep} \cite{smith2006optimizer}, which leads to over-estimation of the optimal value and may damage the performance.

\subsection{Multi-Agent RL}\label{subsec2.2}
The single-agent RL is based on Markov decision process (MDP) theory, while for MARL, it mainly stems from Markov game \cite{shapley1953stochastic}, which generalizes the MDP and was proposed as the standard framework for MARL \cite{littman1994markov}.

We can use a tuple to formalize Markov game, namely $(N, \boldsymbol{S}, \boldsymbol{A}_{1,2,\ldots,N}, r_{1,2,\ldots,N}, p)$, where $N$ being the number of agents in the game system, $\boldsymbol{S} =\{\boldsymbol{s}_{1}, \ldots, \boldsymbol{s}_{n}\}$ is a finite set of system states, $n$ being the number of states in the system, $\boldsymbol{A}_{k}$ is the action set of agent $k\in\{1,\ldots,N\}$; $r_{k} : \boldsymbol{S} \times \boldsymbol{A}_{1} \times \ldots \times \boldsymbol{A}_{N} \times \boldsymbol{S} \rightarrow \mathbb{R}$ is the reward function of agent $k$, determining the immediate reward, $p : \boldsymbol{S} \times \boldsymbol{A}_{1} \times \ldots \times \boldsymbol{A}_{N} \rightarrow \mu(\boldsymbol{S})$ is the transition function. Each agent has its own strategy and chooses actions according to its strategy. Under the joint strategy $\boldsymbol{\pi}\triangleq(\pi_{1}, \ldots, \pi_{N})$, at each time step, the system state is transferred by taking the joint action $\boldsymbol{a}=(a_{1}, \ldots, a_{N})$ selected according to the joint strategy and each agent receives the immediate reward as the consequence of taking the joint action. To measure the performance of a strategy, either the future discounted reward or the average reward over time can be used, depending on the policies of other agents. This results in the following definition for the expected discounted reward for agent $k$ under a joint policy $\boldsymbol{\pi}$ and initial state $\boldsymbol{s}(0)= \boldsymbol{s} \in \boldsymbol{S}$:
\begin{equation}\label{eq11}
V_{k}^{\boldsymbol{\pi}}(s)=\mathrm{E}^{\boldsymbol{\pi}}\left\{\sum_{t=0}^{\infty} \gamma^{t} r_{k}(t+1) | \boldsymbol{s}(0)=\boldsymbol{s}\right\},
\end{equation}
while the average reward for agent $k$ under this joint policy is defined as:
\begin{equation}\label{eq22}
J_{k}^{\boldsymbol{\pi}}(s)=\lim_{T \rightarrow \infty} \frac{1}{T} \mathrm{E}^{\boldsymbol{\pi}}\left\{\sum_{t=0}^{T} r_{k}(t+1) | \boldsymbol{s}(0)=\boldsymbol{s}\right\}.
\end{equation}

On the basis of Eq. (\ref{eq11}) (the most used form), the action-value function $Q_{k}^{\boldsymbol{\pi}} : \boldsymbol{S} \times \boldsymbol{A}_{1} \times \ldots \times \boldsymbol{A}_{N} \rightarrow \mathbb{R}$ of agent $k$ under the joint strategy $\boldsymbol{\pi}$ can be written as follows according to Bellman equation:
\begin{equation}\label{eq33}
Q_{k}^{\boldsymbol{\pi}}(s,\boldsymbol{a})=r_{k}(s,\boldsymbol{a})+\gamma \mathrm{E}_{s^{\prime}\sim p}\left[V_{k}^{\boldsymbol{\pi}}(s^{\prime})\right],
\end{equation}
where $V_{k}^{\boldsymbol{\pi}}(s)=\mathbb{E}_{\boldsymbol{a} \sim \boldsymbol{\pi}}\left[Q_{k}^{\boldsymbol{\pi}}(s, \boldsymbol{a})\right]$ and $s^{\prime}$ is the system state at the next time step.
The commonly used MARL methods are generally based on Q-learning. The general multi-agent Q-learning framework is shown in Algorithm 1.
\begin{algorithm}[!t]
\caption{general multi-agent Q-learning framework}
\LinesNumbered
\KwIn{Initial Q value of all state-action pairs for each agent $k$}
\KwOut{Convergent Q value for each agent $k$}
 Initialize $Q_{k}(s,\boldsymbol{a})=0, \quad \forall s, \boldsymbol{a}, k$ \;
\While{not termination condition}{
    \For{all agents $k$}{
        select action $a_{k}$
    }
    execute joint action $\boldsymbol{a}=(a_{1}, \ldots a_{N})$\;
    observe new state $s^{\prime},$ rewards $r_{k}$\;
    \For{all agents $k$}{
        $Q_{k}(s,\boldsymbol{a})\! =\! (1-\alpha)Q_{k}(s,\boldsymbol{a})\! +\! \alpha\left[r_{k}(s,\boldsymbol{a})\! +\! \gamma V_{k}(s^{\prime})\right]$
    }}
\end{algorithm}

MARL enables each agent to learn the optimal strategy to maximize its cumulative reward. However, the value function of each agent is related to the joint strategy $\boldsymbol{\pi}$ of all agents, so it is in general impossible for all players in a game to maximize their payoff simultaneously.
For MARL, an important solution concept is \emph{Nash equilibrium}. Given these opponent strategies, the \emph{best response} of agent $k$ to a vector of opponent strategies is defined as the strategy $\pi_{k}^{*}$ that achieves the maximum expected reward, which is given as follows:
\begin{equation}\label{eq44}
\mathrm{E}\left\{r_{k} | \pi_{1}, \ldots, \pi_{k}, \ldots, \pi_{N}\right\}\! \leq \!\mathrm{E}\left\{r_{k} | \pi_{1}, \ldots, \pi_{k}^{*}, \ldots, \pi_{N}\right\}, \forall \pi_{k}.
\end{equation}
Then the \emph{Nash equilibrium} is represented by a joint strategy $\boldsymbol{\pi}^{*}\triangleq \left(\pi_{1}^{*}, \ldots, \pi_{N}^{*}\right)$ in which each agent acts with the \emph{best response} $\pi_{k}^{*}$ to others and all other agents follow the joint policy $\boldsymbol{\pi}_{-k}^{*}$ of all agents except $k$, where the joint policy $\boldsymbol{\pi}_{-k}^{*}\triangleq\left(\pi_{1}^{*},\ldots,\pi_{k-1}^{*},\pi_{k+1}^{*},\ldots, \pi_{N}^{*}\right)$. In this case, as long as all other agents keep their policies unchanged, no agent can benefit by changing its policy.
Many MARL algorithms reviewed strive to converge to \emph{Nash equilibrium}. In addition, the Q-function will eventually converge to the \emph{Nash Q-value} $\boldsymbol{Q}^{*}=(Q_{1}^{*},\ldots, Q_{N}^{*})$ received in a \emph{Nash equilibrium} of the game.

\section{Description of The Proposed Method}\label{sec3}
Co-DQL is developed from a new algorithm, called independent double Q-learning method,  which is also firstly proposed in this paper. In the following, we first present the independent double Q-learning method, and then introduce Co-DQL, finally, we analyze its convergence properties.

\subsection{Independent Double Q-learning Method}\label{subsec3.1}
Most MARL methods are based on Q-learning.
However, as described in Section~\ref{subsec2.1}, traditional RL methods cause the problem of over-estimation, which to some extent harms the performance of RL methods.
In \cite{hasselt2010double}, a double Q-learning algorithm is proposed, which uses double estimators instead of $\max _{a} Q_{t}\left(s_{t+1}, a \right)$ to approximate $\max _{a} E\left\{Q_{t}\left(s_{t+1}, a \right)\right\}$,
 which is helpful to avoid the problem of over-estimation in standard Q-learning.

Inspired by independent Q-learning \cite{tan1993multi}, we develop an independent double Q-learning method based on the UCB rule. For each agent $k$, it is associated with two different action-value functions, each of which is updated with a value from the other action-value function for the next state. More specifically, suppose that the two action-value functions are $Q^\mathfrak{a}_k$ and $Q^\mathfrak{b}_k$, and one of them is randomly selected for updating each time. The updating process of the action-value function $Q^\mathfrak{a}_k$ is as follows. Firstly, the maximal valued action $a^{*}_k$ in the next state $s^{\prime}$ is selected according to the action-value function $Q^\mathfrak{a}_k$, namely, $a^{*}_k= \operatorname{argmax} _{a} Q^\mathfrak{a}_k\left(s^{\prime}, a\right)$. Then we use the value $Q^\mathfrak{b}_k\left(s^{\prime}, a^{*}_k\right)$ to update $Q^\mathfrak{a}_k$:
\begin{equation}\label{eq106}
Q^\mathfrak{a}_k(s, a) \leftarrow Q^\mathfrak{a}_k(s, a)+\alpha\left(r_k+\gamma Q^\mathfrak{b}_k\left(s^{\prime}, a^{*}_k\right)-Q^\mathfrak{a}_k(s, a)\right),
\end{equation}
instead of using the value $Q^{\mathfrak{a}}_k\left(s^{\prime}, a^{*}_k\right)=\max _{a} Q^{\mathfrak{a}}_k \left(s^{\prime}, a\right)$ to update $Q^\mathfrak{a}_k$ in independent Q-learning. The updating of $Q^\mathfrak{b}_k$ is similar to this.

Here two multi-layer neural networks are used to fit the two Q functions, which are expressed as $Q^\mathfrak{a}_k\left(s, a; \boldsymbol{\theta}_{t} \right)$ and $Q^\mathfrak{b}_k\left(s, a; \boldsymbol{\theta}_{t}^{\prime} \right)$ respectively. Usually the latter is called target Q-function (or target network). The update mode is similar to the one of deep double Q-learning \cite{van2016deep} and the target value $Y_{k,t} \equiv r_{k, t+1}+\gamma Q^\mathfrak{b}_k(s_{t+1}, \operatorname{argmax}_a Q^\mathfrak{a}_k\left(s_{t+1}, a; \boldsymbol{\theta}_{t}\right), \boldsymbol{\theta}_{t}^{\prime})$.
In order to make the target network update smoother, we adopt the soft target update \cite{lillicrap2015continuous} instead of copying the network weights directly \cite{mnih2015human}:
\begin{equation}\label{addeq66}
\boldsymbol{\theta}^{\prime} \longleftarrow \tau \boldsymbol{\theta}+(1-\tau) \boldsymbol{\theta}^{\prime},
\end{equation}
where $\tau \ll 1$. The soft update method makes the weights of the target Q-function change slowly, so does the target values. Compared with the direct copy of the weights, the soft update method can enhance the learning stability \cite{mnih2015human}.

To balance exploration and exploitation, the UCB exploration strategy is used to  select an action to be performed by the agent $k$:
\begin{equation}\label{eq77}
a_k = \underset{c \in \boldsymbol{A}_k}{\operatorname{argmax} }\left\{Q_{k}(s_{k}, c)+ \sqrt{\frac{\ln R_{s_{k}}}{R_{s_{k}, c}}}\right\},
\end{equation}
where $R_{s_{k}}$ denotes the number of times  state $s_{k}$ has been visited and $R_{s_{k}, c}$ denotes the number of times   action $c$ has been chosen in this state until now. If action $c$ has been chosen rarely in some states, then the second term will dominate the first term and action $c$ will be explored. As learning progresses, the first term dominates the second term and the UCB strategy  ultimately becomes a greedy one.
Although $\epsilon$-greedy strategy is easier to implement for problems with larger state space, we prefer the UCB strategy if possible, since in the preliminary test we observe that the UCB strategy is slightly better than the $\epsilon$-greedy strategy \cite{prabuchandran2014multi}. From the perspective of exploration mechanism, the exploratory action selection for UCB is based on both the learnt Q-values and the number of times an action has been chosen in the past, hence it tends to be more inclined to explore those actions that are rarely explored.

In this method, agent $k$ just regards other agents as a part of the environment.
Therefore, this method ignores the dynamic resulting from the actions of the other agents  and the convergence is not guarantee.
In order to learn a better cooperative strategies and make learning process more stable and robust, we introduce Co-DQL, which uses mean field approximation, a new reward allocation mechanism and local state sharing method.

\subsection{Cooperative Double Q-learning Method }\label{subsec3.2}

With the number of agents increasing, the dimension of joint action $\boldsymbol{a}$ increases exponentially, so when the number of agents is relatively large, it is often not feasible to directly calculate the joint action function $Q_{k}(s, \boldsymbol{a})$ for each agent $k$.
Mean field approximation is first proposed in \cite{yang2018mean} to deal with the problem.
Its core idea is that the interactions within the population of agents are approximated by those between an agent and the average  of its neighboring agents \footnote{The neighborhood size is a user-specific parameter. It can take a value from [1, $N$] where $N$ is the total number of agents.}.
Specifically, a very natural approach is to decompose the joint action-value function as follows:
\begin{equation}\label{eq55}
Q_{k}(s_{k}, \boldsymbol{a})=\mathrm{E}_{l\sim d} \left[Q_{k}(s_{k}, a_{k}, a_{l})\right],
\end{equation}
where $d$ is the uniform distribution on the index set $\mathcal{N}(k)$ which is the set of the neighboring agents of agent $k$ and the size of the index set is $N_{k}=|\mathcal{N}(k)| $.
Suppose that each agent has $C$ discrete actions $\{1, 2, \ldots, C\}$.
Then the action $a_k$ of agent $k$ can be coded using one-hot, namely, $a_k \triangleq \left[a_{k,1},a_{k,2}, \ldots, a_{k,C}\right]$, where each component corresponds to a possible action, and obviously at any time only one component is one and the others are zero.
Hence the mean action $\overline{a}_{k}$ can be expressed as: $\overline{a}_{k} \triangleq\left[\overline{a}_{k,1},\overline{a}_{k,2}, \ldots, \overline{a}_{k,C}\right]$, where each component $\overline{a}_{k,i} = \mathrm{E}_{l\sim d} \left[a_{l,i}\right]$ for $i\in\{1, 2, \ldots, C\}$, simply recorded as $\overline{a}_{k}=\mathrm{E}_{l\sim d} \left[a_{l}\right]$. Intuitively, $\overline{a}_{k}$ can be seen as the empirical distribution of the actions taken by the neighbors of agent $k$ \cite{yang2018mean}.
Naturally, there is the following relationship between the one-hot coding action $a_{l}$ of   agent $l$ and the mean action:
\begin{equation}\label{eqadd1}
  a_{l}=\overline{a}_{k}+ \delta_{l, k},
\end{equation}
where $\delta_{l, k}$ is a small fluctuation.
Under the premise of twice-differentiable, using Taylor expansion theory, the mean field approximation is expressed by the following formulation on the basis of Eq.~\ref{eq55}:
\begin{equation}\label{eq66}
\begin{split}
  Q_{k}& (s_{k}, \boldsymbol{a}) =\mathrm{E}_{l\sim d} \left[Q_{k}(s_{k}, a_{k}, a_{l})\right] \\
  &=\mathrm{E}_{l\sim d} [Q_{k}(s_{k}, a_{k}, \overline{a}_{k})+\nabla Q_{k}(s_{k}, a_{k}, \overline{a}_{k})\cdot \delta_{l, k}\\
  &\quad +\frac{1}{2}\delta_{l, k}\cdot \nabla^{2} Q_{k}(s_{k}, a_{k}, \xi_{l,k})\cdot\delta_{l, k}]\\
  &=Q_{k}(s_{k}, a_{k}, \overline{a}_{k})+\nabla Q_{k}(s_{k}, a_{k}, \overline{a}_{k})\cdot \mathrm{E}_{l\sim d}[\delta_{l, k}]\\
  &\quad +\frac{1}{2}\mathrm{E}_{l\sim d}[\delta_{l, k}\cdot \nabla^{2} Q_{k}(s_{k}, a_{k}, \xi_{l,k})\cdot\delta_{l, k}]\\
  &=Q_{k}(s_{k}, a_{k}, \overline{a}_{k})+\frac{1}{2}\mathrm{E}_{l\sim d}[R_k(a_l)]\\
  & \approx Q_{k}(s_{k}, a_{k},\overline{a}_{k}),
\end{split}
\end{equation}
where $\mathrm{E}_{l\sim d}[\delta_{l, k}]=0$ is easily known from Eq.~\ref{eqadd1}, and $R_k(a_l)\triangleq \delta_{l, k}\cdot \nabla^{2} Q_{k}(s_{k}, a_{k}, \xi_{l,k})\cdot\delta_{l, k}$ denotes the Taylor polynomial's remainder with $\xi_{l,k}=\overline{a}_{k}+\epsilon_{l,k}\cdot\delta_{l, k}$ and $\epsilon_{l,k}\in[0,1]$ \cite{yang2018mean}. Under some mild conditions, it can be proved that $R_k(a_l)$ is a random variable close to zero and can be omitted\cite{yang2018mean}.
For large-scale TSC, this way of implicit modeling   the behavior of other agents has great advantages, which makes the input dimension of each agent $k$'s Q-function drastically reduce, and the joint action dimension decreases from $C^{N_k}$ to constant $C^2$. It is worth noting that we only need to pay attention to the actions of the current time step, rather than the historical  behavior of the neighbors. This is mainly due to the fact that the traffic state dynamics is Markovian, which will be further discussed in Section~\ref{subsec4.1}.

For partially observable Markov traffic scenarios, each agent $k$ can get its own reward $r_k$ and local observation $s_k$ at each time step.
The goal of MARL in cooperative situation  is to maximize the global benefits or minimize the regrets \footnote{In this paper, regrets refer to the waiting time of vehicles, the length of queues, etc.}. However, there may be the so-called credit assignment problem \cite{agogino2008analyzing}  in MARL, so each agent often does not directly regard the global reward as its reward. Instead, we set each agent to maintain its own reward. In addition, if each agent only considers its own immediate reward, then the agent may become selfish, which may be harmful to cooperation. Based on the above considerations, we propose to allocate each agent's reward according to the following formulation:
\begin{equation}\label{eq111}
\hat{r}_k = r_k + \alpha \cdot \sum_{i \in \mathcal{N}(k)}r_i,
\end{equation}
where $\alpha\in[0,1]$ is a discount factor that can be flexibly used to balance selfishness and cooperation.
If $\alpha$ is set to 0, then each signal agent only considers the immediate reward of its own intersection, greedily maximizing
the throughput of its own intersection, which may damage the global reward of the road network; if $\alpha$ is set to 1,
this means that each agent may get the global reward and suffers from credit assignment problem as described earlier.
Specifically, we make $0<\alpha<1$. The idea behind is as follows: For each signal agent $k$, despite the action
selection may be not always beneficial to the neighboring agents, the reallocated reward received after an action depends
on its own immediate reward and the immediate reward of the neighboring agents. Once the immediate rewards of the
neighboring agents are low, the second term of Eq.~\ref{eq111} will take a small value which means the action taken by signal
agent $k$ may be not so great for the neighboring agents. While higher immediate rewards of the neighboring agents
will encourage signal agent $k$ and accordingly the second term of Eq.~\ref{eq111} will take a larger value. This reward allocation
mechanism in Eq.~\ref{eq111} in turn affects the action selection of agent $k$, with the aim of maximizing the global reward of
the road network.
The reward allocation mechanism is similar to the one mentioned in  \cite{chu2019multi}, but we do not strictly limit the distance between agent $k$ and the neighboring agents.

The local state sharing method is described below. For agent $k$, the average of the local state of its neighboring agents is taken as the additional input of agent $k$'s action-value function. Hence, the state of   agent $k$ can be represented as:
\begin{equation}\label{eq112}
\hat{s}_k = \langle s_k, \frac{1}{N_{k}}\sum_{i \in \mathcal{N}(k)}s_i \rangle,
\end{equation}
where $\hat{s}_k$ represents agent $k$'s joint state.
This method implicitly shares state information among agents, and if the dimension of local state is assumed to be $|s|$, its joint dimension is constant $|s|^2$, which is independent of the number of agents.

Based on the above introduction, \emph{Cooperative double Q-learning} (Co-DQL) algorithm is proposed.
Compared with the centralized control method\cite{casas2017deep}\cite{lowe2017multi}, this algorithm reduces the joint input dimension of action-value function from $C^{N_k}\cdot|s|^{N_k}$ to $C^2\cdot|s|^2$ at the cost of a small amount of communication and calculation\cite{yang2018mean}, which avoids the  curse of dimension in large-scale problems.
The pseudo code of Co-DQL is given in Algorithm 2.
In this algorithm, multi-layer perceptions parameterized by $\phi$ and $\phi_{-}$ are used to represent the two action-value functions of each agent. Co-DQL works as follows:
\begin{description}
    \item [\textbf{Step 0}]\textbf{ Initialize:} For each $k=1, \ldots, N$, initialize neural network parameters $\phi_k$, $\phi_{-,k}$ and initialize mean action $\overline{a}_k$ for agent $k$.

    \item [\textbf{Step 1}]\textbf{ Check the termination condition:} If a problem-specific stopping condition is met, stop and save the training neural network model.

     \item[\textbf{Step 2}]\textbf{ Select action:} For each $k=1, \ldots, N$, according to the current observation $\hat{s}_k$ of agent $k$, select action $a_k$ under the UCB policy.

     \item[\textbf{Step 3}]\textbf{ Execute action:} For each $k=1, \ldots, N$, agent $k$ executes action $a_k$ (all agents execute action synchronously), gets immediate reward $r_k$ and next state observation $s^{\prime}_k$.
	
     \item[\textbf{Step 4}]\textbf{ Obtain samples:} For each $k=1, \ldots, N$, compute the mean action $\overline{a}_k$, reward $\hat{r}_k$ after reallocation and next local state $\hat{s}^{\prime}_k$ after sharing.

     \item[\textbf{Step 5}]\textbf{ Store samples in buffer:} For each $k=1, \ldots, N$, store the results of step 4 as a tuple sample $\left\langle \hat{\boldsymbol{s}}, \boldsymbol{a}, \boldsymbol{\hat{r}}, \hat{\boldsymbol{s}}^{\prime}, \overline{\boldsymbol{a}} \right\rangle$ in replay buffer $\mathcal{D}_k$;
         If the number of samples stored in the $\mathcal{D}_k$ is less than the minimum number of samples required for training, goto Step 1, otherwise the next step is executed sequentially.

     \item[\textbf{Step 6}]\textbf{ Compute sample target values:} For each $k=1, \ldots, N$, $M$ samples are randomly extracted from $\mathcal{D}_k$ and the target value $Y_{k}^{\mathrm{Co-DQL}}$ is calculated according to the sample data.

     \item[\textbf{Step 7}]\textbf{ Update Neural Network Parameters:} For each $k=1, \ldots, N$, the gradient of the parameter $\phi_k$ is obtained from the loss function, and $\phi_k$ is updated according to the learning rate, then $\phi_{-,k}$ is softly updated with update rate $\tau$. Goto Step 1. \\
\end{description}

For most RL algorithms, the termination condition is generally set to be that the number of episodes experienced by agents reachs the preset number. The preset number of episodes is usually selected according to the training situation of the algorithm in the given problem.

The action-value function $Q^\mathfrak{a}_{k}(\cdot|\phi)$ (parameterized by $\phi$) is trained by minimizing the loss:
\begin{equation}\label{eq88}
\ell(\phi_{k})=\left(Q^\mathfrak{a}_{k}(\hat{s}_{k}, a_{k}, \overline{a}_{k}; \phi)-Y_{k}^{\mathrm{Co-DQL}}\right)^{2},
\end{equation}
where $Y_{k}^{\mathrm{Co-DQL}}$ is the target value of agent $k$ and is calculated by the following formulation:
\begin{equation}\label{eq99}
Y_{k}^{\mathrm{Co-DQL}}\!=\!\hat{r}_{k}\!+\gamma Q^\mathfrak{b}_{k}({\hat{s}_{k}}^{\prime}, \operatorname{argmax}_{a_k} \!Q^\mathfrak{a}_{k}(\hat{s}^{\prime}_{k},\! a_k,\! \overline{a}_{k};\! \phi),\overline{a}^{\prime}_{k};\! \phi_{-}),
 \end{equation}

In Co-DQL, the mean field approximation makes every independent agent  learn the awareness of collaboration with the others. Moreover, the reward allocation mechanism and the local state sharing method of agents improve the stability and robustness of the training process compared with the independent agent learning method.

In order to theoretically support the effectiveness of our proposed Co-DQL algorithm, we provide the convergence proof  under some assumptions in the next subsection.
\begin{algorithm}[!t]
\caption{Co-DQL}
\LinesNumbered
\KwIn{Initial parameters $\phi$ and mean action $\overline{a}$ for all agents}
\KwOut{Parameters $\phi$ for all agents}
 Initialize $Q^\mathfrak{a}_{k}(\cdot|{\phi}), Q^\mathfrak{b}_{k}(\cdot|{\phi_{-}})$ and $\overline{a}_{k}$ for all $k \in\{1,\ldots,N\}$ \\
\While{not termination condition}{
    For each agent $k$ , select action $a_{k}$ using the UCB exploration strategy from Eq.~\ref{eq77}\\
    Take the joint action $\boldsymbol{a}=(a_{1}, \ldots a_{N})$ and observe the reward $\boldsymbol{r}=(r_{1}, \ldots, r_{N})$ and the next observations $\boldsymbol{s}^{\prime}=(s^{\prime}_{1}, \ldots, s^{\prime}_{N})$ \\
    Compute $\overline{\boldsymbol{a}}, \boldsymbol{\hat{r}}$, $\hat{\boldsymbol{s}}$ and $\hat{\boldsymbol{s}}^{\prime}$ \\
    Store $\left\langle \hat{\boldsymbol{s}}, \boldsymbol{a}, \boldsymbol{\hat{r}}, \hat{\boldsymbol{s}}^{\prime}, \overline{\boldsymbol{a}} \right \rangle$ in replay buffer $\mathcal{D}$ \\
    \For{$k=1$ to $N$}{
        Sample $M$ experiences $\left\langle \hat{\boldsymbol{s}}, \boldsymbol{a}, \boldsymbol{\hat{r}}, \hat{\boldsymbol{s}}^{\prime}, \overline{\boldsymbol{a}} \right\rangle$ from $\mathcal{D}$\\
        Compute target value $Y_{k}^{\mathrm{Co-DQL}}$ by Eq.~\ref{eq99}\\
        Update the $Q$ network by minimizing the loss
        $$\mathcal{L}(\phi_{k})=\frac{1}{M}\sum\left(Q^\mathfrak{a}_{k}(\hat{s}_{k},a_{k},\overline{a}_{k};\phi)-Y_{k}^{\mathrm{Co-DQL}}\right)^{2}$$
    }
    Update the parameters of the target network for each agent $k$ with updating rate $\tau$
    $$\phi_{k}^{-} \leftarrow \tau \phi_{k}+(1-\tau) \phi_{_,k}$$
    }
\end{algorithm}

\subsection{Convergence Analysis}\label{subsec3.3}

In previous literature, the convergence of mean field Q-learning under the set of tabular Q-functions and the convergence of when Q-function is represented by other function approximators have been proved \cite{li2019efficient} \cite{yang2018mean}. Under similar constraints, we develop the convergence proof of Co-DQL, which is the mean field RL with double estimators.

Assuming that there are only a limited number of state-action pairs, for each agent $k$, we can write updating rules of two functions $Q^\mathfrak{a}_{k}$ and $Q^\mathfrak{b}_{k}$ of agent $k$ according to Section~\ref{subsec3.1} and Section~\ref{subsec3.2}:
\begin{equation}\label{eq107}
  \begin{aligned}
&Q^\mathfrak{a}_{k}(s,a_k,\overline{a}_{k}) \!\leftarrow\! (1 \!- \!\alpha) Q^\mathfrak{a}_{k}(s,a_k,\overline{a}_{k}) \! + \!\alpha (r \! + \!\gamma Q^\mathfrak{b}_{k}(s^{\prime}, \mathfrak{a}^{*}_k,\overline{a}_{k}))\\
&Q^\mathfrak{b}_{k}(s,a_k,\overline{a}_{k}) \!\leftarrow\! (1 \!- \!\alpha) Q^\mathfrak{b}_{k}(s,a_k,\overline{a}_{k}) \! + \!\alpha (r  \!+ \!\gamma Q^\mathfrak{a}_{k}(s^{\prime}, \mathfrak{b}^{*}_k,\overline{a}_{k})),
  \end{aligned}
\end{equation}
where $\mathfrak{a}^{*}_k = \operatorname{argmax}_{a_k} Q^\mathfrak{a}_{k}(s^{\prime}, a_k, \overline{a}_{k})$, and $\mathfrak{b}^{*}_k = \operatorname{argmax}_{a_k} Q^\mathfrak{b}_{k}(s^{\prime}, a_k, \overline{a}_{k})$. At any update time step, either of the two of Eq.~\ref{eq107} is updated. Our goal is to prove that both $\boldsymbol{Q}^\mathfrak{a} = (Q^\mathfrak{a}_1, \ldots, Q^\mathfrak{a}_N)$ and $\boldsymbol{Q}^\mathfrak{b} = (Q^\mathfrak{b}_1, \ldots, Q^\mathfrak{b}_N)$ converge to Nash Q-values.
Our proof follows the convergence proof framework of single agent Double Q-learning \cite{hasselt2010double}, and we use the following assumptions and lemma.

\begin{assumption}
Each action-value pair is visited infinitely often, and the reward is bounded by some constant $K$ .
\end{assumption}
\begin{assumption}
Agent's policy is Greedy in the Limit with Infinite Exploration (GLIE). In the case with the Boltzmann
policy, the policy becomes greedy w.r.t. the Q-function in the limit as the temperature decays asymptotically to zero.
\end{assumption}
\begin{assumption}
For each stage game $[Q_{t}^{1}(s), \ldots, Q_{t}^{N}(s)]$ at time $t$ and in state $s$ in training, for all $t, s, j \in\{1, \ldots, N\}$, the Nash equilibrium $\boldsymbol{\pi}_{*}=[\pi_{*}^{1}, \ldots, \pi_{*}^{N}]$ is recognized either as 1) the global optimum or 2) a saddle point expressed as:
\begin{enumerate}
  \item $\mathbb{E}_{\pi_{*}}[Q_{t}^{j}(s)] \geq \mathbb{E}_{\pi}[Q_{t}^{j}(s)], \forall \pi \in \Omega (\prod_{k} \mathcal{A}^{k})$;
  \item
      $\mathbb{E}_{\pi_{*}}[Q_{t}^{j}(s)] \geq \mathbb{E}_{\pi^j} \mathbb{E}_{\pi_{*}^{-j}}[Q_{t}^{j}(s)], \forall \pi^{j} \in \Omega (\mathcal{A}^{j})$ and \\
      $\mathbb{E}_{\pi_{*}}[Q_{t}^{j}(s)] \leq \mathbb{E}_{\pi_{*}^{j}} \mathbb{E}_{\pi^{-j}}[Q_{t}^{j}(s)], \forall \pi^{-j} \in \Omega (\prod_{k\neq j} \mathcal{A}^{k}).$
\end{enumerate}
\end{assumption}
\begin{lemma}
The random process $\left\{\Delta_{t}\right\}$ defined in $\mathbb{R}$ as
$$\Delta_{t+1}(x)=\left(1-\alpha_{t}(x)\right) \Delta_{t}(x)+\alpha_{t}(x) F_{t}(x)$$
converges to zero with probability 1 (w.p.1) when
\begin{enumerate}
  \item $0 \leq \alpha_{t}(x) \leq 1, \sum_{t} \alpha_{t}(x)=\infty, \sum_{t} \alpha_{t}^{2}(x)<\infty$;
  \item $x \in \mathscr{X},$ the set of possible states, and $|\mathscr{X}|<\infty$;
  \item $\left\|\mathbb{E}\left[F_{t}(x) | \Im_{t}\right]\right\|_{W} \leq \gamma\left\|\Delta_{t}\right\|_{W}+c_{t},$ where $\gamma \in[0,1)$ and $c_{t}$ converges to zero w.p.1;
  \item $\operatorname{var}\left[F_{t}(x)|\Im_{t}\right] \leq K(1+\left\|\Delta_{t}\right\|_{W}^{2})$ with constant $K>0$.
\end{enumerate}
Here $\mathscr{F}_{t}$ denotes the filtration of an increasing sequence of $\sigma$-fields including the history of processes; $\alpha_{t}, \Delta_{t}, F_{t} \in \mathscr{F}_{t}$ and $\|\cdot\|_{W}$ is a weighted maximum norm [30].
\end{lemma}
\begin{proof}
Similar to the proof of Theorem 1 in \cite{jaakkola1994convergence} and Corollary 5 in \cite{szepesvari1999unified}.
\end{proof}
Our theorem and proof sketches are as follows:
\begin{theorem}
In a finite-state stochastic game, if Assumption 1,2 \& 3, and Lemma 1's first and second conditions are met, then both $\boldsymbol{Q}^\mathfrak{a}$ and $\boldsymbol{Q}^\mathfrak{b}$ as updated by the rule of Algorithm 2 in Eq.~\ref{eq107} will converge to the Nash Q-value $\boldsymbol{Q}^{*}=(Q_{1}^{*},\ldots, Q_{N}^{*})$ with probability one.
\end{theorem}
\begin{proof}
We need to show that the third and fourth conditions of Lemma 1 hold so that we can apply it to prove Theorem 1.
Obviously, the updates of functions $\boldsymbol{Q}^\mathfrak{a}$ and $\boldsymbol{Q}^\mathfrak{b}$ are symmetrical, so as long as one of them is proved to converge, the other must converge.
By subtracting  two sides of Eq.~\ref{eq107} by $\boldsymbol{Q}^{*}$, and then the following formula can be obtained by comparing with the equation in Lemma 1:
\begin{equation}
\begin{aligned}
 & \boldsymbol{\Delta}_{t}(s, \boldsymbol{a})=\boldsymbol{Q}^\mathfrak{a}_{t}(s, \boldsymbol{a})-\boldsymbol{Q}_{*}(s, \boldsymbol{a}) \\
 & \boldsymbol{F}_{t}(s_{t},\boldsymbol{a}_{t})=\boldsymbol{r}_{t}+\gamma \boldsymbol{Q}^\mathfrak{b}_t\left(s_{t+1},\mathfrak{a}^{*}\right)-\boldsymbol{Q}_{*}\left(s_{t},\boldsymbol{a}_{t}\right)\label{eq108},
\end{aligned}
\end{equation}
where $\mathfrak{a}^{*}=\operatorname{argmax}_{\boldsymbol{a}} \boldsymbol{Q}^\mathfrak{a}(s_{t+1},\boldsymbol{a}_{t}, \overline{\boldsymbol{a}}_{t})$. Let $\Im_{t}=\{\boldsymbol{Q}^\mathfrak{a}_0,\boldsymbol{Q}^\mathfrak{b}_0, s_0, \boldsymbol{a}_0, \alpha_{0}, \boldsymbol{r}_{1}, s_{1}, \dots, s_{t}, \boldsymbol{a}_{t}\}$ denote the $\sigma$-fields generated by all random variables in the history of the stochastic game up to time $t$. Note that $\boldsymbol{Q}^\mathfrak{a}_t$ and $\boldsymbol{Q}^\mathfrak{b}_t$ are two random variables derived from the historical trajectory up to time $t$. Given the fact that all $\boldsymbol{Q}^\mathfrak{a}_\tau$ and $\boldsymbol{Q}^\mathfrak{b}_\tau$ with $\tau< t$ are $\mathscr{F}_t$ -measurable, both $\boldsymbol{\Delta}_{t}$ and $\boldsymbol{F}_{t}$ are therefore also $\mathscr{F}_t$ -measurable.

Since the reward is bounded by some constant $K$ in Assumption 1, then $\operatorname{Var}[\boldsymbol{r}_t]<\varpropto$, the fourth condition in the lemma holds.

Next, we show that the third condition of the lemma holds. We can rewrite Eq.~\ref{eq108} as follows:
\begin{equation}\label{eq109}
  \boldsymbol{F}_{t}(s_{t},\boldsymbol{a}_{t})=\boldsymbol{F}^Q_{t}(s_{t},\boldsymbol{a}_{t})+ \gamma (\boldsymbol{Q}^\mathfrak{b}_t(s_{t+1},\mathfrak{a}^{*}) - \boldsymbol{Q}^\mathfrak{a}_t(s_{t+1},\mathfrak{a}^{*})),
\end{equation}
where $\boldsymbol{F}^Q_{t}=\boldsymbol{r}_{t}+\gamma \boldsymbol{Q}^\mathfrak{a}_t \left(s_{t+1}, \mathfrak{a}^{*}\right) -\boldsymbol{Q}^{*} \left(s_{t}, \boldsymbol{a}_{t}\right)$ is the value of $\boldsymbol{F}_{t}$ if normal MF-Q would be under consideration. In [17], $\|\mathbb{E}[\boldsymbol{F}^Q_{t}|\Im_{t}]\|_{W}\leq \gamma\left\|\Delta_{t}\right\|_{W}$ has been proved, so in order to meet the third condition, we identify $c_t =\gamma (\boldsymbol{Q}^\mathfrak{b}_t \left(s_{t+1}, \mathfrak{a}^{*} \right) -\boldsymbol{Q}^\mathfrak{a}_t \left(s_{t+1}, \mathfrak{a}^{*}\right))$ and it is sufficient to show that $\boldsymbol{\Delta}_{t}^{\mathfrak{b} \mathfrak{a}} =\boldsymbol{Q}_{t}^\mathfrak{b} -\boldsymbol{Q}_{t}^\mathfrak{a}$ converges to zero. The update of $\boldsymbol{\Delta}_{t}^{\mathfrak{b}\mathfrak{a}}$ depends on whether $\boldsymbol{Q}^\mathfrak{b}$ or $\boldsymbol{Q}^\mathfrak{a}$ is updated, so
\begin{equation}
  \begin{aligned}\label{eq110}
&\boldsymbol{\Delta}_{t+1}^{\mathfrak{b}\mathfrak{a}}(s_t,\boldsymbol{a}_t) = \boldsymbol{\Delta}_{t}^{\mathfrak{b}\mathfrak{a}}(s_t, \boldsymbol{a}_t) + \alpha_t \boldsymbol{F}^\mathfrak{b}_{t}(s_t,\boldsymbol{a}_t) ,\text{ or}\\
&\boldsymbol{\Delta}_{t+1}^{\mathfrak{b}\mathfrak{a}}(s_t,\boldsymbol{a}_t) = \boldsymbol{\Delta}_{t}^{\mathfrak{b}\mathfrak{a}}(s_t, \boldsymbol{a}_t) - \alpha_t \boldsymbol{F}^\mathfrak{b}_{t}(s_t,\boldsymbol{a}_t),
  \end{aligned}
\end{equation}
where $\boldsymbol{F}^\mathfrak{a}_{t}(s_t, \boldsymbol{a}_t)= \boldsymbol{r}_{t}+\gamma \boldsymbol{Q}^\mathfrak{b}_t\left(s_{t+1}, \mathfrak{a}^{*}\right) - \boldsymbol{Q}^\mathfrak{a}_t \left(s_{t}, \boldsymbol{a}_{t} \right)$ and $\boldsymbol{F}^\mathfrak{b}_t(s_t, \boldsymbol{a}_t) = \boldsymbol{r}_{t}+\gamma \boldsymbol{Q}^\mathfrak{a}_t \left(s_{t+1}, \mathfrak{b}^{*}\right) - \boldsymbol{Q}^\mathfrak{b}_t \left(s_{t}, \boldsymbol{a}_{t} \right)$. We define $\xi^{\mathfrak{b}\mathfrak{a}}_t=\frac{1}{2}\alpha_t$, then
\begin{equation*}
 \begin{split}
  \mathbb{E}&[\boldsymbol{\Delta}^{\mathfrak{b}\mathfrak{a}}_{t+1}(s_t,\boldsymbol{a}_t)|\Im_{t}] = \boldsymbol{\Delta}^{\mathfrak{b}\mathfrak{a}}_{t}(s_t,\boldsymbol{a}_t)
     + \mathbb{E}[\alpha_t \boldsymbol{F}^\mathfrak{b}_{t}(s_t,\boldsymbol{a}_t)\\
  &\quad - \alpha_t \boldsymbol{F}^\mathfrak{a}_{t}(s_t,\boldsymbol{a}_t)|\Im_{t}]\\
  &= \boldsymbol{\Delta}^{\mathfrak{b}\mathfrak{a}}_{t}(s_t,\boldsymbol{a}_t)+ \mathbb{E}[\alpha_t \gamma (\boldsymbol{Q}^\mathfrak{a}_t\left(s_{t+1},\mathfrak{b}^{*}\right) -\boldsymbol{Q}^\mathfrak{b}_t\left(s_{t+1},\mathfrak{a}^{*}\right))\\
  &\quad -\alpha_t (\boldsymbol{Q}^\mathfrak{b}_t\left(s_{t},\boldsymbol{a}_{t}\right) -\boldsymbol{Q}^\mathfrak{a}_t\left(s_{t},\boldsymbol{a}_{t}\right))|\Im_{t}]\\
  &= (1- \xi^{\mathfrak{b}\mathfrak{a}}_t (s_t,\boldsymbol{a}_t)) \boldsymbol{\Delta}^{\mathfrak{b}\mathfrak{a}}_{t} (s_t, \boldsymbol{a}_t)\\
  &\quad +\xi^{\mathfrak{b}\mathfrak{a}}_t(s_t, \boldsymbol{a}_t) \mathbb{E}[\boldsymbol{F}^{\mathfrak{b}\mathfrak{a}}_{t} (s_t, \boldsymbol{a}_t)|\Im_{t}],
\end{split}
\end{equation*}
where $\mathbb{E}[\boldsymbol{F}^{\mathfrak{b}\mathfrak{a}}_{t}(s_t,\boldsymbol{a}_t)|\Im_{t}]\!=\!\gamma \mathbb{E}[\boldsymbol{Q}^\mathfrak{a}_t\left(\! s_{t+1},\! \mathfrak{a}^{*}\! \right)\!- \! \boldsymbol{Q}^\mathfrak{b}_t\left(\! s_{t+1},\! \mathfrak{a}^{*}\!\right)\! | \Im_{t}]$. At each time step, one of the following two cases must hold.

Case 1: $\mathbb{E}[\boldsymbol{Q}^\mathfrak{a}_t\left(s_{t+1},\! \mathfrak{b}^{*} \right)\!|\Im_{t}]\! \geq\! \mathbb{E}[\boldsymbol{Q}^\mathfrak{b}_t\left(s_{t+1},\! \mathfrak{a}^{*}\right)\!|\Im_{t}]$. We have $\boldsymbol{Q}^\mathfrak{a}_t\left(\!s_{t+1}, \! \mathfrak{a}^{*}\!\right)\!=\!\max \boldsymbol{Q}^\mathfrak{a}_t\left(\! s_{t+1}, \! \boldsymbol{a} \!\right) \! \geq \! \boldsymbol{Q}^\mathfrak{a}_t\left(\! s_{t+1},\! \mathfrak{b}^{*}\! \right),$ therefore
\begin{equation*}
  \begin{aligned}
  |\mathbb{E}[\boldsymbol{F}^{\mathfrak{b}\mathfrak{a}}_{t}(s_t,\boldsymbol{a}_t)|\Im_{t}]| & = \gamma \mathbb{E}[\boldsymbol{Q}^\mathfrak{a}_t(s_{t+1},\mathfrak{b}^{*}) -\boldsymbol{Q}^\mathfrak{b}_t(s_{t+1}, \mathfrak{a}^{*})|\Im_{t}]\\
  & \leq \gamma \mathbb{E}[\boldsymbol{Q}^\mathfrak{a}_t(s_{t+1},\mathfrak{a}^{*})-\boldsymbol{Q}^\mathfrak{b}_t(s_{t+1}, \mathfrak{a}^{*})|\Im_{t}]\\
  & \leq\|\boldsymbol{\Delta}^{\mathfrak{b}\mathfrak{a}}_{t}\|.
   \end{aligned}
\end{equation*}
Case 2: $\mathbb{E}[\boldsymbol{Q}^\mathfrak{a}_t\left(s_{t+1},\mathfrak{b}^{*} \right)|\Im_{t}]<\mathbb{E}[\boldsymbol{Q}^\mathfrak{b}_t\left(s_{t+1}, \mathfrak{a}^{*}\right)|\Im_{t}]$. We have $\mathbb{E}[\boldsymbol{Q}^\mathfrak{b}_t\left(s_{t+1}, \mathfrak{b}^{*} \right)|\Im_{t}] \geq \mathbb{E}[\boldsymbol{Q}^\mathfrak{b}_t\left(s_{t+1}, \mathfrak{a}^{*}\right)|\Im_{t}]$. Then
\begin{equation*}
  \begin{aligned}
  |\mathbb{E}[\boldsymbol{F}^{\mathfrak{b}\mathfrak{a}}_{t}(s_t,\boldsymbol{a}_t)|\Im_{t}]| & = \gamma \mathbb{E}[\boldsymbol{Q}^\mathfrak{b}_t(s_{t+1}, \boldsymbol{e}^{*})- \boldsymbol{Q}^\mathfrak{a}_t(s_{t+1}, \mathfrak{b}^{*})|\Im_{t}]\\
  & \leq \gamma \mathbb{E}[\boldsymbol{Q}^\mathfrak{b}_t(s_{t+1}, \mathfrak{b}^{*})-\boldsymbol{Q}^\mathfrak{a}_t(s_{t+1},\mathfrak{b}^{*})|\Im_{t}]\\
  & \leq\|\boldsymbol{\Delta}^{\mathfrak{b}\mathfrak{a}}_{t}\|.
   \end{aligned}
\end{equation*}
Hence, no matter which of the above cases is hold, we can obtain the satisfactory result, that is, $|\mathbb{E}[\boldsymbol{F}^{\mathfrak{b}\mathfrak{a}}_{t}(s_t, \boldsymbol{a}_t)|\Im_{t}]| \leq \|\boldsymbol{\Delta}^{\mathfrak{b}\mathfrak{a}}_{t}\|$. Then, we can apply Lemma 1 and get the convergence of $\boldsymbol{\Delta}_{t}^{\mathfrak{b}\mathfrak{a}}$ to 0, the third condition is thus hold. Finally, with all conditions are satisfied, Theorem 1 is proved.
\end{proof}

\section{Application of Co-DQL to TSC}\label{sec4}
\begin{figure}[!t]
\centering{\includegraphics[width=0.8\linewidth]{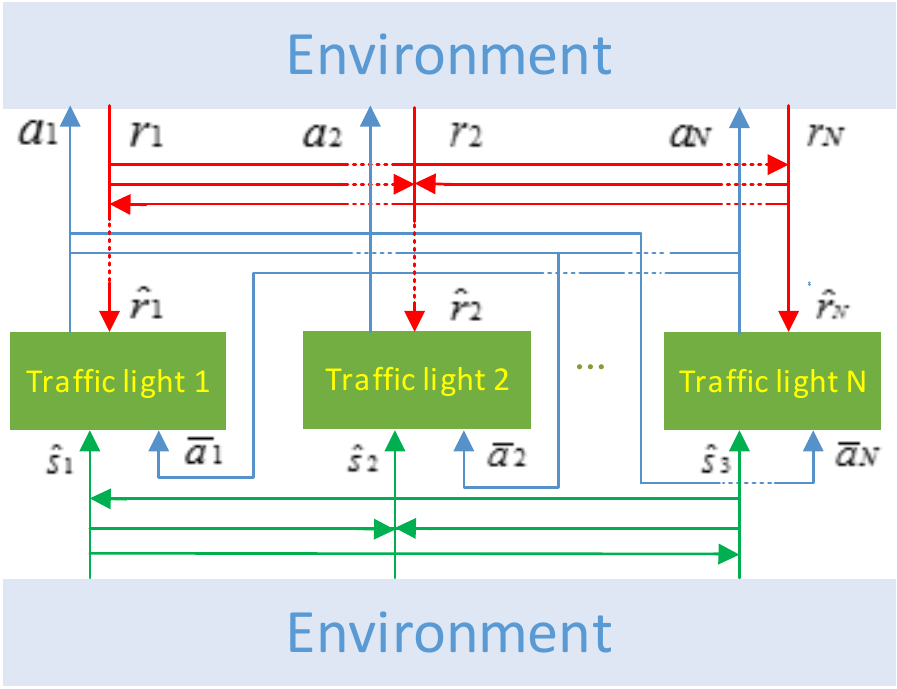}}
\caption{The architecture diagram of Co-DQL for TSC. For each $k=1,\dots,N$, $\hat{s}_k$ denotes the local state information after sharing, $\bar{a}_k$ represents the mean action information, $a_k$ means the action will be executed, ${r}_k, \hat{r}_k$ represent the immediate reward before and after reallocation, respectively.
The red arrow, blue arrow and green arrow represent the transfer of reward information, action information and state information, respectively.}\label{fig11}
\end{figure}
\begin{figure*}[!t]
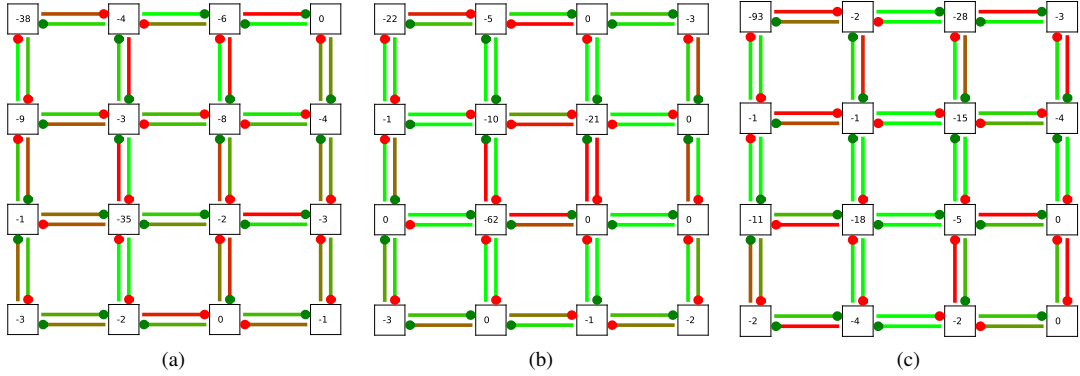

    \centering
	  \subfloat[]{\includegraphics[width=0.25\linewidth]{cut_flow1.pdf}} \label{fig22a}\quad
	  \subfloat[]{\includegraphics[width=0.25\linewidth]{cut_flow2.pdf}} \label{fig22b}\quad
	  \subfloat[]{\includegraphics[width=0.25\linewidth]{cut_flow3.pdf}} \label{fig22c}\hfill
	  \caption{Illustration of the grid traffic signal system simulator. (a) global random traffic flow, (b) double-ring traffic flow, and (c) four-ring traffic flow. Each rectangle represents a signalized intersection and each two adjacent intersections are connected by two one-way lanes. The color of each lane implies the level of congestion and the number in the rectangle represents the immediate reward for the intersection. }
	  \label{fig22}
	\end{figure*}
This section first uses MDP notations to represent the key elements of TSC problem, so that MARL can be used in TSC. To facilitate the training and evaluation of the MARL model applied to TSC problem, we also introduce the TSC simulators.

\subsection{Description of TSC Based on MDP Notations}\label{subsec4.1}
Although we model the entire traffic network in a decentralized way as a multi-agent structure, the global state of the whole traffic system is still Markov, namely, the next state only depends on the current state:
\begin{equation}\label{eq100}
s_{t+1}=f(s_{t}, \boldsymbol{a}_{t}),
\end{equation}
where $s_{t}$ and $s_{t+1}$ denotes the state of traffic system at time step $t$ and $t+1$, $\boldsymbol{a}_{t}$ denotes the joint action of the traffic system at time step $t$. Therefore, it can be modeled using the framework of MARL described in Section~\ref{subsec2.2}.

When to cope with TSC problem, there are many different MDP settings. Their differences lie in the definition of action space, state space or reward function, etc.
\cite{prashanth2010reinforcement}\cite{chu2019multi}\cite{aslani2017adaptive} \cite{el2013multiagent} \cite{jin2015adaptive} \cite{tan2019cooperative}. Here, we focus on the following two kinds of MDP settings. Note that it may be potential to extend our method to other kinds of settings. Since the source code of our method is open \footnote{https://github.com/Brucewangxq/larger\_real\_net}, an interested reader can try to test or extend it to deal with other kinds of MDP settings.

\subsubsection{A simplified MDP setting for TSC problem}\label{subsec4.1.1}
Suppose a road network has $N$ signalized intersections, i.e., $N$ signal agents.
The action of signal agent $k$ at time step $t$ can be written as $a_{k,t}$, and its local observation or state is $s_{k,t}$. We set the signal agent's actions at each intersection has only two possible cases $\{0,1\}$: Green traffic lights for incoming traffic in the north and south directions and red traffic lights in the east and west directions at the same time, or contrary to that, so the action space is $\{0,1\}^{N}$.
The local state, which is the observation vector $s_{k,t}$, is the waiting queue density (or queue length) on all the one-way lanes (or edges) connected to the intersection $k$: $s_{k,t}=[q[kn], q[ks], q[kw], q[ke]]$, where $q[kn], q[ks], q[kw]$ and $q[ke]$ represent the waiting queue density in four directions related to intersection $k$ respectively, and they are the lanes of vehicles driving in the direction of intersection $k$. The value space of each of them can be expressed as $\{0,1,2, \ldots, \max_{q}\}$, where  $\max_{q}$ is the maximum capacity of vehicles on an lane between every two intersections. For the peripheral signal agent of the   system, if there is no road connected to it in a certain direction, the number of vehicles in that direction is always zero. For simplicity, it is assumed that a normally traveling vehicle has the same speed and can start or stop immediately.

For any signal agent $k$, the reward at time step $t$ can be calculated by the number of vehicles waiting on all lanes towards the intersection, that is,
\begin{equation}\label{eq101}
 r_{k,t} =-\sum_{j \in \{n,s,w,e\}}\left|q_t[k j]\right|,
\end{equation}
where   $q_t[k j]$ is the number of vehicles that have zero speed on lane $kj$ leading to intersection $k$. To avoid changing traffic signal too frequently, the action can be taken every $\Delta t$ time steps, that is, a Markov state transition occurs only once every $\Delta t$ time steps. Then from the $T$-th to $T + 1$-th state transition, the signal agent obtains the sum of the rewards in $\Delta t$ time steps, that is,
\begin{equation}\label{eq102}
\boldsymbol{R}_{T}=\sum_{t=(T-1) \Delta t}^{T \Delta t-1} \boldsymbol{r}_{t}\left(s_{t}, \boldsymbol{a}_{t}\right),
\end{equation}

Our goal is to minimize the total waiting time of vehicles in the traffic network:
\begin{equation}\label{eq103}
\operatorname{\max}_{\boldsymbol{\pi}} J=\mathbb{E}\left[\sum_{T=1}^{T_{max}} \gamma^{T-1}\left(\frac{1}{\Delta t} \sum_{t=(T-1) \Delta t}^{T \Delta t-1} \boldsymbol{r}_{t}\left(s_{t}, \boldsymbol{a}_{t}\right)\right)\right],
\end{equation}
where $T_{max}$ denotes the total number of state transitions, and the joint action $\boldsymbol{a}$ changes every $\Delta t$ time steps.

In this simplified situation, all other agents are treated as the neighboring agents of each agent. Fig.~\ref{fig11} shows how to apply Co-DQL to TSC. The input information of each agent includes the shared local state information and the mean action information calculated from actions of the neighboring agents  in the previous time step. Each agent receives a reallocated reward after performing an action.

\subsubsection{A more realistic MDP setting for TSC problem}\label{subsec4.1.2}
In the literature of RL for TSC, there are several standard action definitions, such as phase duration\cite{aslani2017adaptive}, phase switch\cite{el2013multiagent} \cite{jin2015adaptive} and phase itself\cite{tan2019cooperative}\cite{chu2019multi}\cite{prashanth2010reinforcement}. Here, we follow the last definition and pre-define a set of feasible phases for each signal agent.
Specifically, we adopt the definition of feasible phases in \cite{chu2019multi}, which defines five feasible phases for each signal agent, including east-west straight, east-west left-turn, and three straight and left-turn for east, west and north-south. These five feasible phases constitute the action space, each phase corresponds to an action.
Each signal agent selects one of them to implement for a duration of $\Delta t$ at each Markov time step.
In addition, a yellow time $t_y<\Delta t$ is enforced after each phase switch to ensure safety.

After comprehensively understanding a variety of commonly used state definitions\cite{tan2019cooperative} \cite{aslani2017adaptive} \cite{chu2019multi}, we tend to follow the one in \cite{chu2019multi} and define  local state as
\begin{equation}\label{addeq103}
s_{k,t}=\{wait_{k,t}[lane],wave_{k,t}[lane]\},
\end{equation}
where $lane$ is each incoming lane of intersection $k$. $wait$ measures the cumulative delay [s] of the first vehicle and $wave$ measures the total number [veh] of approaching vehicles along each incoming lane. In our experiment, we use \verb"laneAreaDetector" in Simulation of Urban Mobility (SUMO)\cite{chu2016large} \cite{codeca2018monaco} to obtain the state information, and in practice, the state information can be obtained by near-intersection induction-loop detectors as described in [23].

Similar to the definition of reward in the simplified TSC problem mentioned earlier, we also further consider the cumulative delay of the first car as a regularizer:
\begin{equation}\label{addeq104}
r_{k,t} =-\sum_{lane}\left|q_{k,t+\Delta t}[lane]+\beta \cdot wait_{k,t+\Delta t}[lane]\right|,
\end{equation}
where $\beta$ is the regularization rate and typically chosen to approximately scale different reward terms into the same range. Note that the rewards are only measured at time $t+\Delta t$. Compared to other reward definitions such as wave\cite{aslani2017adaptive} and appropriateness of green time\cite{teo2014agent}, the reward we defined emphasizes traffic congestion and travel delay, and it is directly correlated to state and action\cite{chu2019multi}.

\subsection{Description of the Simulation Platform}\label{subsec4.2}
\subsubsection{A simplified TSC simulator}\label{subsec4.2.1}
\begin{table}[!t]
\renewcommand{\arraystretch}{1.3}
\caption{Parameter Settings for Simulator}
\label{tab1}
\centering
\begin{tabular}{|c|c|}
\hline
Parameter Type & Value [\emph{unit of measure}]\\
\hline
Normal driving time between two nodes & 5 [\emph{t}]\\
\hline
Initial vehicles in simulator & 100 [\emph{veh}]\\
\hline
New vehicles added &5;4;3 [\emph{veh/t}]\\
\hline
Shortest route length & 2 [\emph{n}]\\
\hline
Longest route length & 20 [\emph{n}]\\
\hline
Signal agent action time interval & 4 [\emph{t}]\\
\hline
Initial random seed number & 10 \\
\hline
\end{tabular}
\vspace{5pt}

\emph{t} means discrete time step, \emph{veh} is the abbreviation of vehicle, \emph{n} denotes node, i.e. intersection. 5;4;3 [\emph{veh/t}] means that the number of new vehicles added per time step is optional and can be set to 5,4 or 3 as needed.
\end{table}

The simulation platform used in Section~\ref{subsec5.2} is a grid TSC system based on OpenAI-gym \cite{brockman2016openai}.
There are three different scenarios in the experiment: global random traffic flow, double-ring traffic flow and four-ring traffic flow which correspond to the three subfigures of Fig.~\ref{fig22} respectively.

Each rectangle denotes a signalized intersection and the number in the rectangle represents the immediate reward for the intersection. Every two adjacent intersections are connected by two one-way lanes. The color of each lane in the picture ranges from green to red, which vaguely means the number of vehicles waiting (at zero speed) on the lane, i.e. the level of congestion. Green means unimpeded and red indicates serious congestion.
During the operation of the simulator, a certain number of vehicles will be generated at each time step and scattered randomly in the road network. And every newly generated vehicle will have a randomly generated route according to a certain rule, and the vehicle will follow the route and finally the vehicle will be removed from the road network when it reaches the destination.

Among these three scenarios, the one difference is that the  rules of generating a driving route of a vehicle, which results in different level of congestion at different intersections. This can simulate the real information of the traffic flow between the main and secondary roads in the city.
In the actual traffic network, serious congestion does often occur only in certain specific sections.
The other difference is that the number of new vehicles added per time step is various, which can be used to simulate different levels of traffic congestion.

The primary parameters of the simulator are listed in Table~\ref{tab1}. The normal driving time between two intersections, that is, the distance between two intersections, indicates that normal driving vehicles need 5 time steps from one intersection to an adjacent intersection.
The initial number (note that it is not the number after resetting the simulator when training model) of vehicles in simulator is used to obtain random seeds.
The shortest route length is 2, which means that the shortest distance that a vehicle generated in the simulator can travel is two intersections. The longest route length is 20, which means that the longest distance that a vehicle generated in the simulator can travel is twenty intersections.
The action time interval of  signal agent  is 4, which means that a signal agent must keep at least 4 time steps before it can change one action.

\subsubsection{A more realistic TSC simulator}\label{subsec4.2.2}
 \begin{figure}[!t]
\centering{\includegraphics[width=6cm,height=5cm]{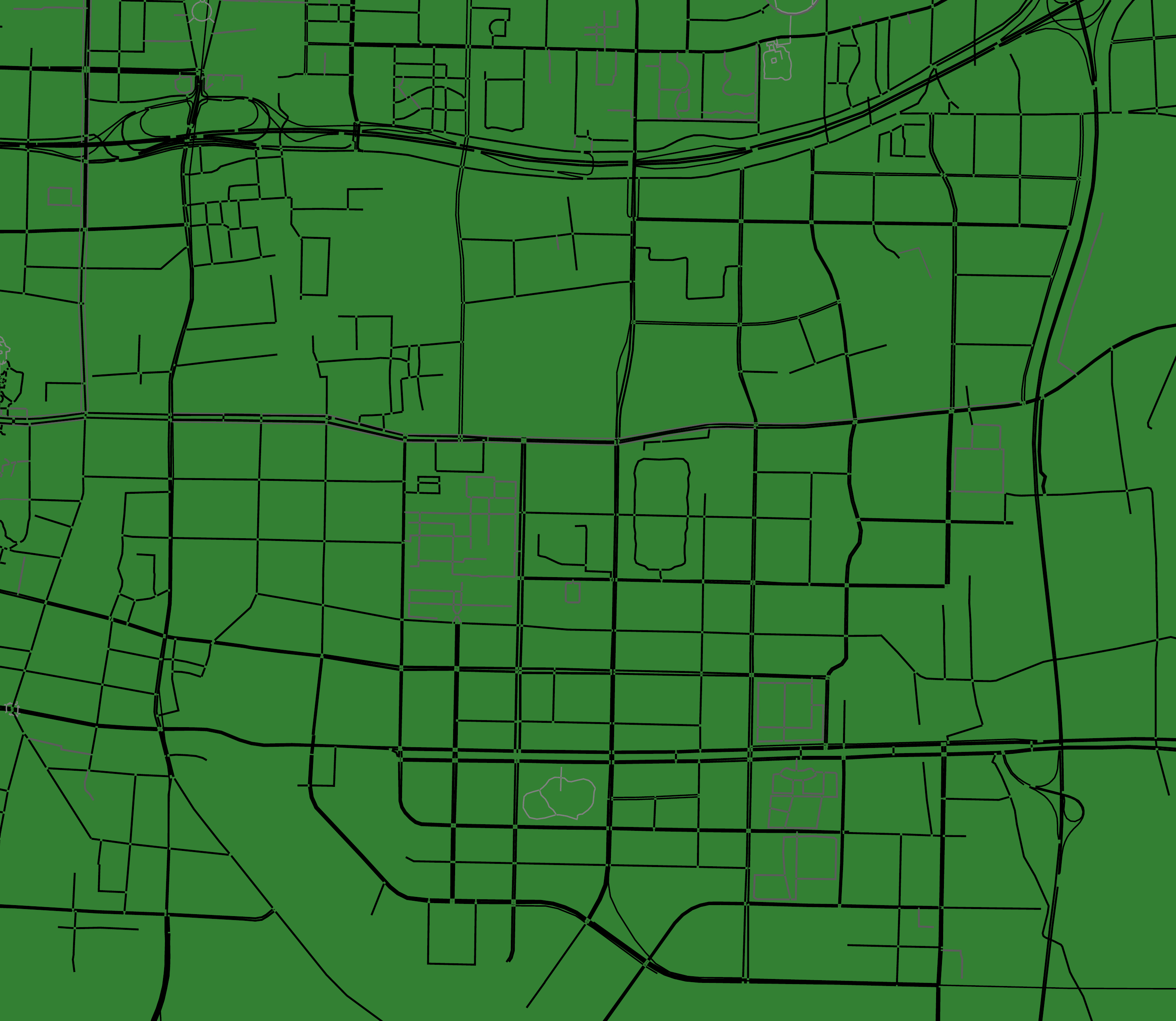}}
\caption{Overall view of the realistic road network with asymmetric geometry. \label{fig888}}
\end{figure}
\begin{figure}[t]
\centering{\includegraphics[width=0.6\linewidth]{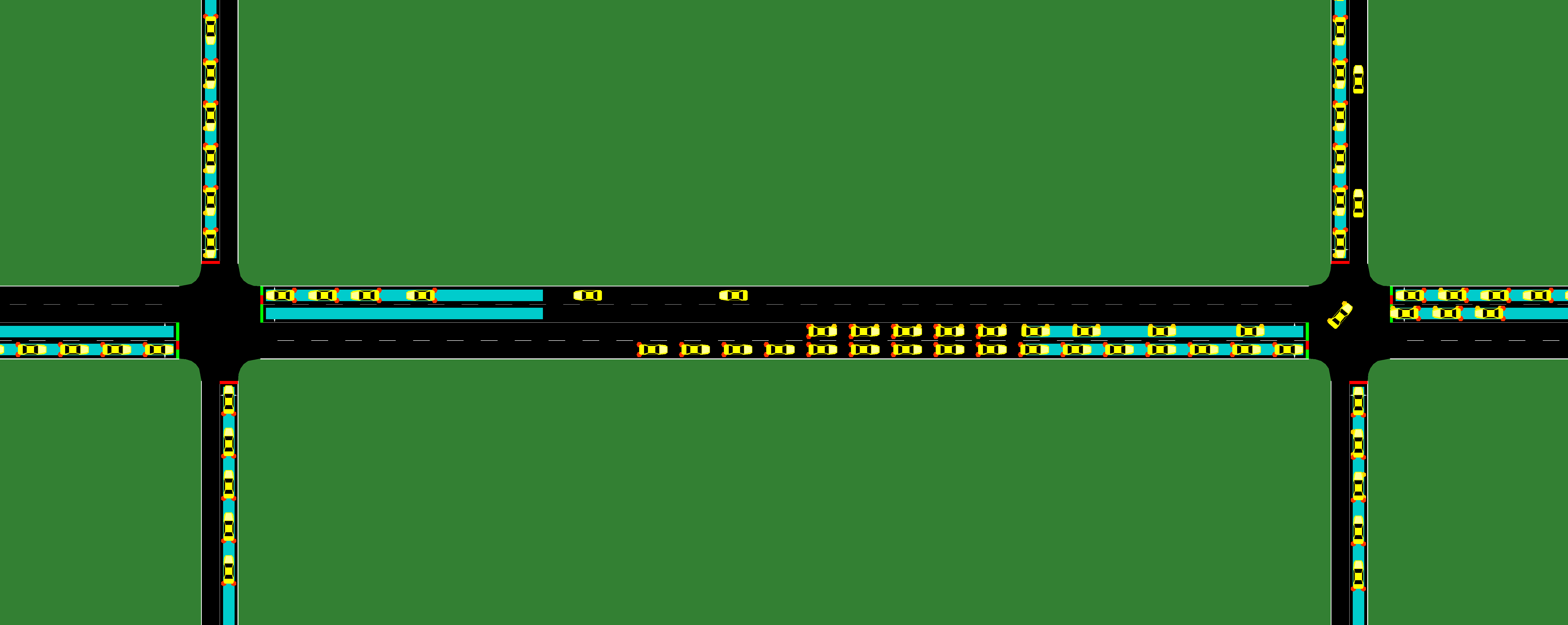}}
\caption{Local view of two adjacent intersections of the realistic road network. \label{fig2222b}}
\end{figure}

We take the road network in some areas of Xi'an as the prototype of the real road network to design a TSC simulator based on SUMO, which has 49 signalized intersections on the real road network.
Fig.~\ref{fig888} and Fig.~\ref{fig2222b} show the overall road network view and a local view of two adjacent intersections, respectively. The cars driving on the road network have the following properties: the length is $5m$, the acceleration is $5m/s$, and the deceleration is $10m/s$.  As for the setting of signal agents' action time interval $\Delta t$, as discussed in \cite{chu2019multi}, if $\Delta t$ is too long, signal agent will not be adaptive enough, if $\Delta t$ is too short, the agent decision will not be delivered on time due to computational cost and communication latency, and it may be unsafe since the action is switched too frequently. Some recent works suggested $\Delta t= 10s, t_y = 5s$ \cite{aslani2017adaptive}, $\Delta t= 5s, t_y = 2s$\cite{chu2019multi}. We adopt the latter setting in the simulator to ensure that each signal agent is more adaptive.

In order to evaluate the robustness and optimality of algorithms in  a challenging TSC scenario, we design intensive, stochastic, time-variant traffic flows to simulate the peak-hour traffic, instead of fixed congestion levels  in the simplified TSC simulator.
The simulation time of each episode is $60min$ and we set up four traffic flow groups.
Specifically, four traffic flow groups are generated as multiples of ``unit" flows $1100veh / hr$, $660veh / hr$, $920veh / hr$, and $552veh / hr$. The first two traffic flows are simulated during the first $40min$, as $[0.4, 0.7, 0.9, 1.0, 0.75, 0.5, 0.25]$ unit flows with $5min$ intervals, while the last two traffic flows are generated during a shifted time window from $15min$ to $55min$, as $[0.3, 0.8, 0.9, 1.0, 0.8, 0.6, 0.2]$ unit flows with $5min$ intervals.

\section{Numerical Experiments and Discussions}\label{sec5}

\subsection{Implementation Details of Algorithms}\label{subsec5.1}
In order to analyze the performance of the proposed algorithm, we compared it with several popular RL methods in the same traffic scenarios. Details of the implementation of Co-DQL  and the other  methods  are described as follows:

Co-DQL: The  procedure described in Section~\ref{subsec3.2} is implemented. Multilayer fully connected neural network is used to approximate the Q-function of each agent. We use the ReLU-activation between hidden layers, and transform the final output of Q-network with it. All agents share the same Q-network, the shared Q-network takes an agent embedding as input and computes Q-value for each candidate action. We also feed in the action approximation $\overline{a}_{k}$ and sharing joint state $\hat{s}_k$. We use the Adam optimizer with a learning rate of 0.0001.
The discounted factor $\gamma$ is set to 0.95, the mini-batch size is 1024, and the reward allocation factor $\alpha$ is set to $1/n$, where $n$ represents the number of neighbor agents. The size of replay buffer is $5 \times 10^{5}$ and $\tau=0.01$ for updating the target networks. The network parameters will be updated once an episode samples are added to the replay buffer.

Multi-Agent A2C (MA2C): The start-of-the-art MARL (decentralized) algorithms for large-scale TSC. The hyper-parameters of the algorithm in the experiment are basically consistent with the original one \cite{chu2019multi}.

Independent Q-learning (IQL): It has almost the same hyper-parameters settings as Co-DQL. And the network architecture is identical to Co-DQL, except a mean action and sharing joint state are not fed as an addition input to the Q-network.

Independent double Q-learning (IDQL): The parameter setting of this method is almost the same as that of independent Q-learning. The main difference is that the double estimators are used when calculating the target value.

Deep deterministic policy gradient (DDPG): This is an off-policy algorithm too. It consists of two parts: actor and critic. Each agent is trained with DDPG algorithm and we share the critic among all agents in each experiment and all of the actors are kept separate. It uses the Adam optimizer with a learning rate of 0.001 and 0.0001 for critics and actors respectively. The settings of other parameters are  the same as those of Co-DQL.

It is noteworthy that all the hyper-parameter settings of all algorithms may affect the performance of the algorithm to a certain extent.

\subsection{Experiments in The Simplified TSC Simulator}\label{subsec5.2}
 By training and evaluating the proposed method in different traffic scenarios, we can demonstrate that the proposed method is promising.
Next, we will analyze the performance of the algorithms in three scenarios.
\begin{figure*}[!t] 
    \centering
	  \subfloat[]{\includegraphics[width=0.30\linewidth]{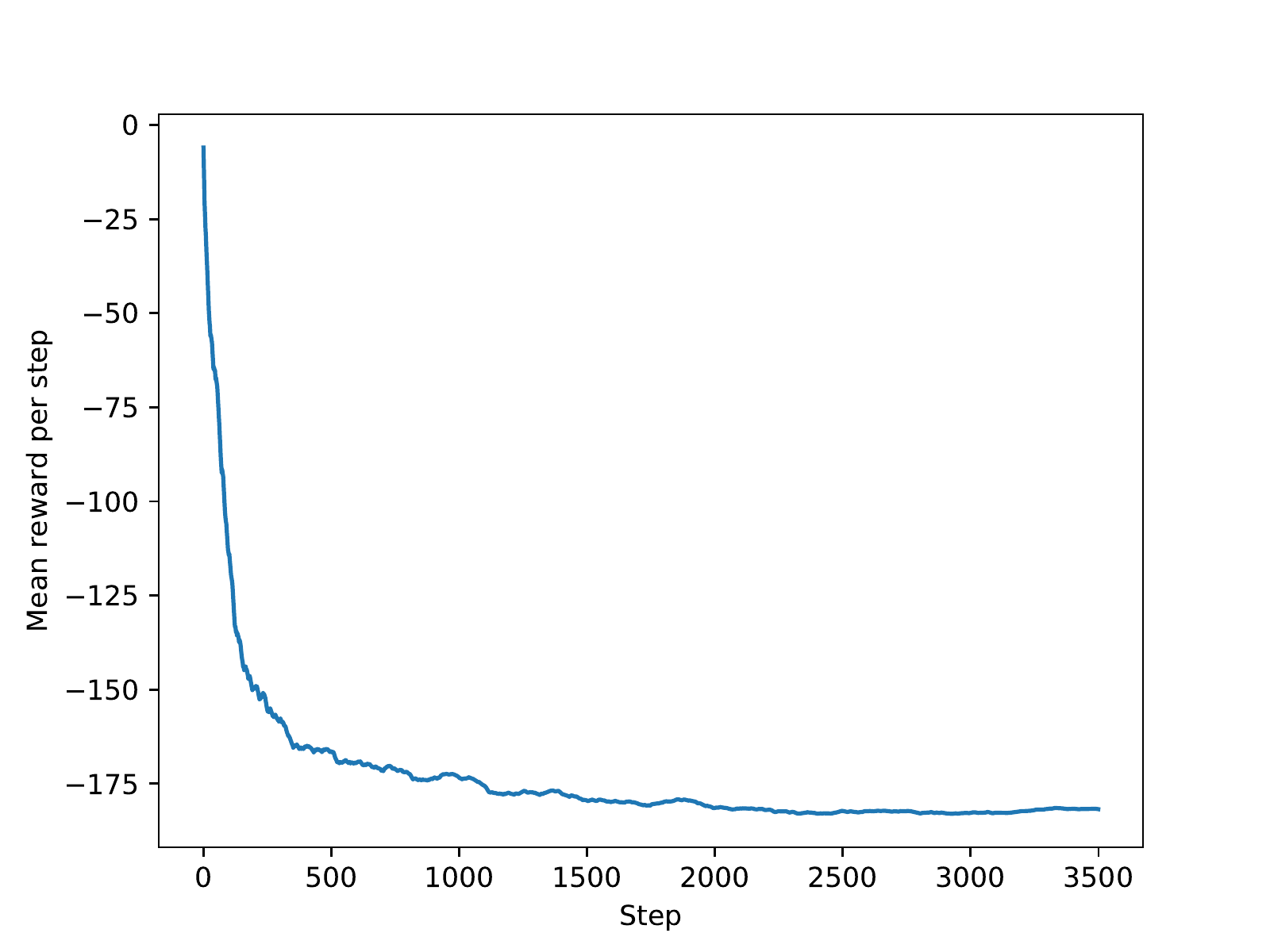}} \label{fig33a}\quad\quad
	  \subfloat[]{\includegraphics[width=0.30\linewidth]{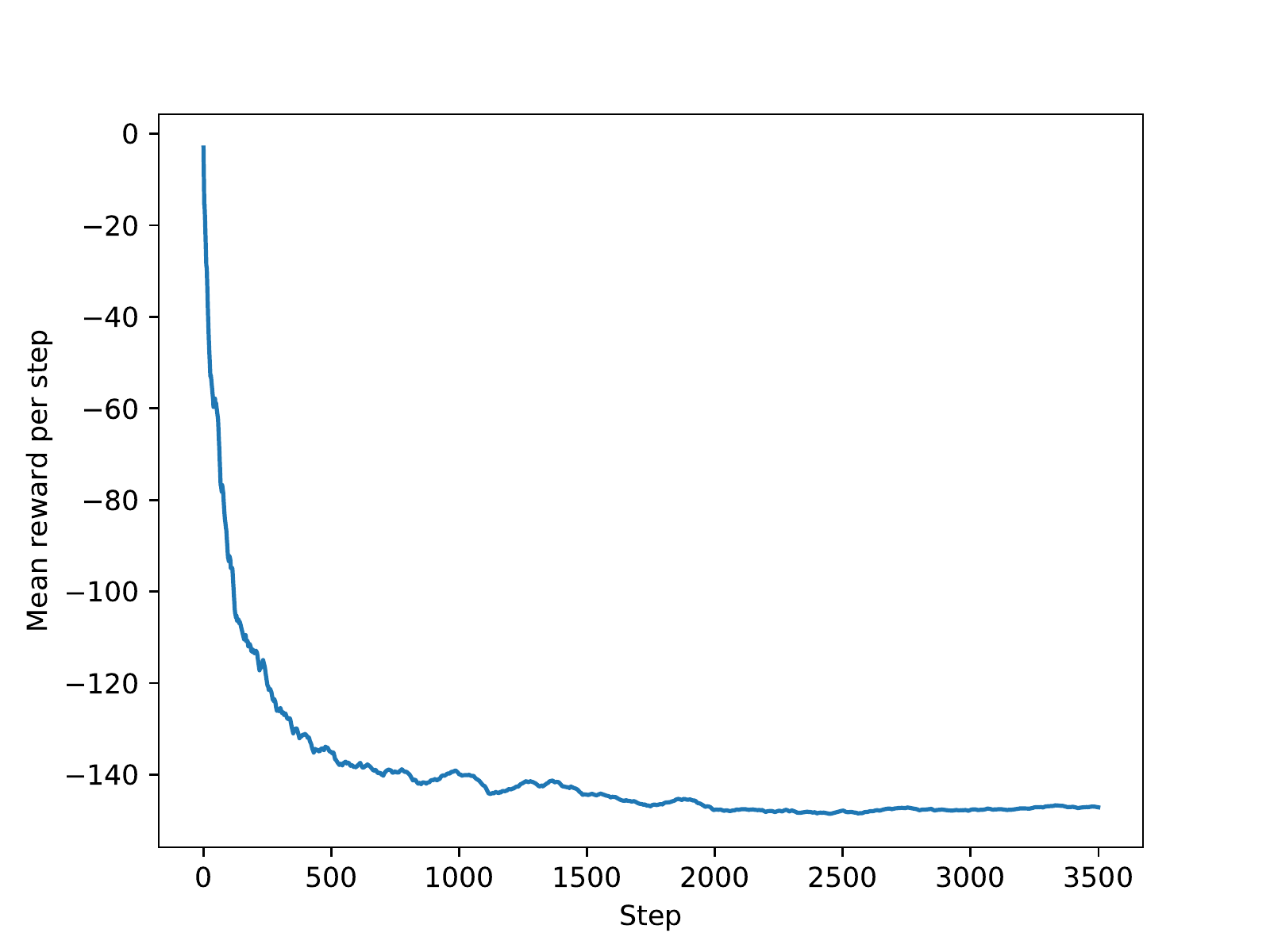}} \label{fig33b}\quad\quad
	  \subfloat[]{\includegraphics[width=0.30\linewidth]{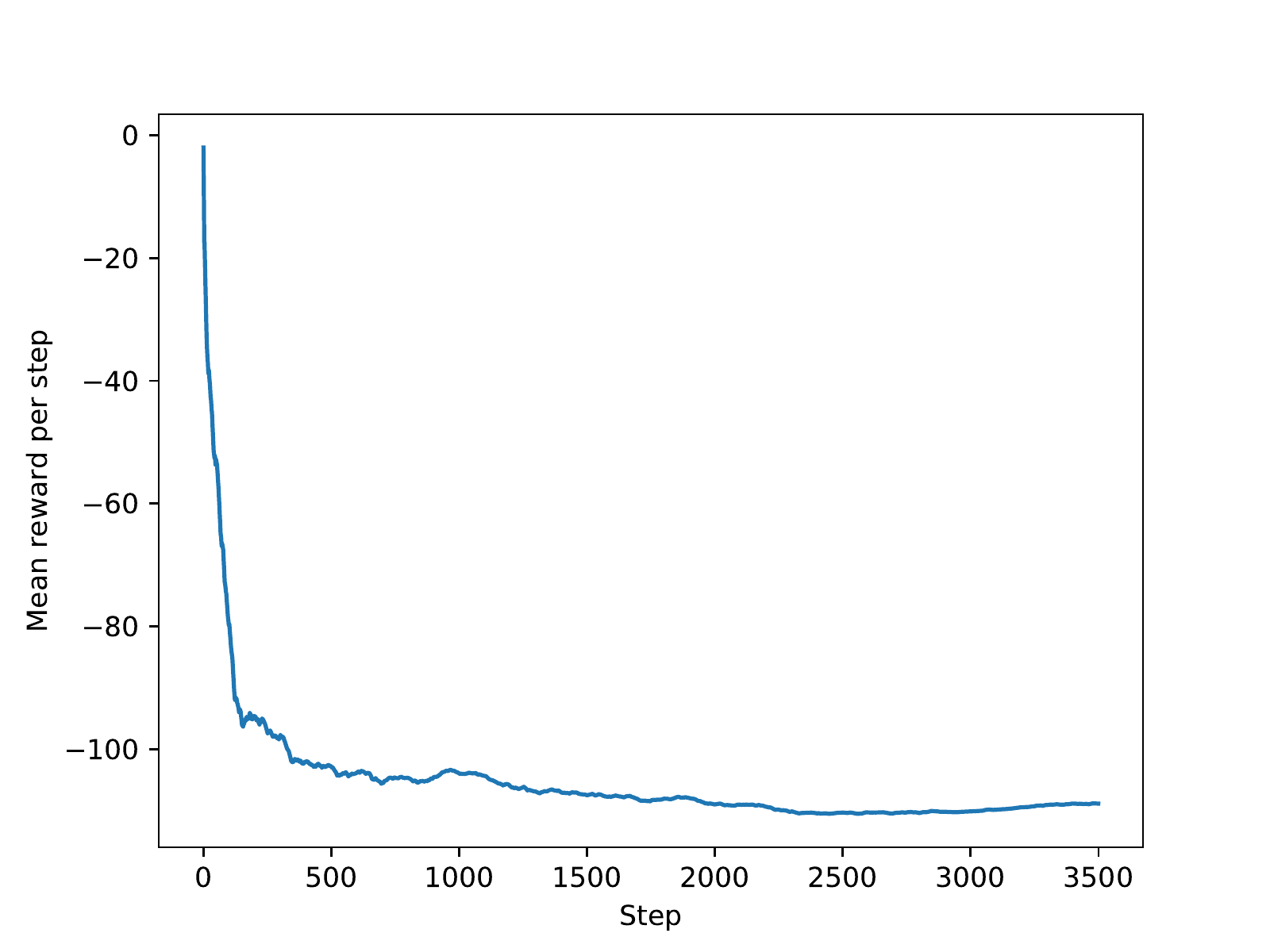}} \label{fig33c}\hfill
	  \caption{Illustration of mean reward change curve of signal agents using random strategy in various traffic flows scenarios. (a) global random traffic flow scenario, (b) double-ring traffic flow scenario, and (c) four-ring traffic flow scenario. At the beginning, there are fewer vehicles running in the simulator. As vehicles are added to the simulator at each time step, there are more and more vehicles in the road network, and the mean reward of signal agents is getting smaller. As the vehicle arriving at the destination will be removed from the simulator, the level of congestion will reaches a stable range. We intercept a certain number of simulator states after stabilization as the selectable initial state of the simulator when training and evaluating MARL models.}
	  \label{fig33}
	\end{figure*}

\subsubsection{global random traffic flow}\label{subsec5.2.1}
\begin{figure}[!t]
\centering{\includegraphics[width=8cm,height=5cm]{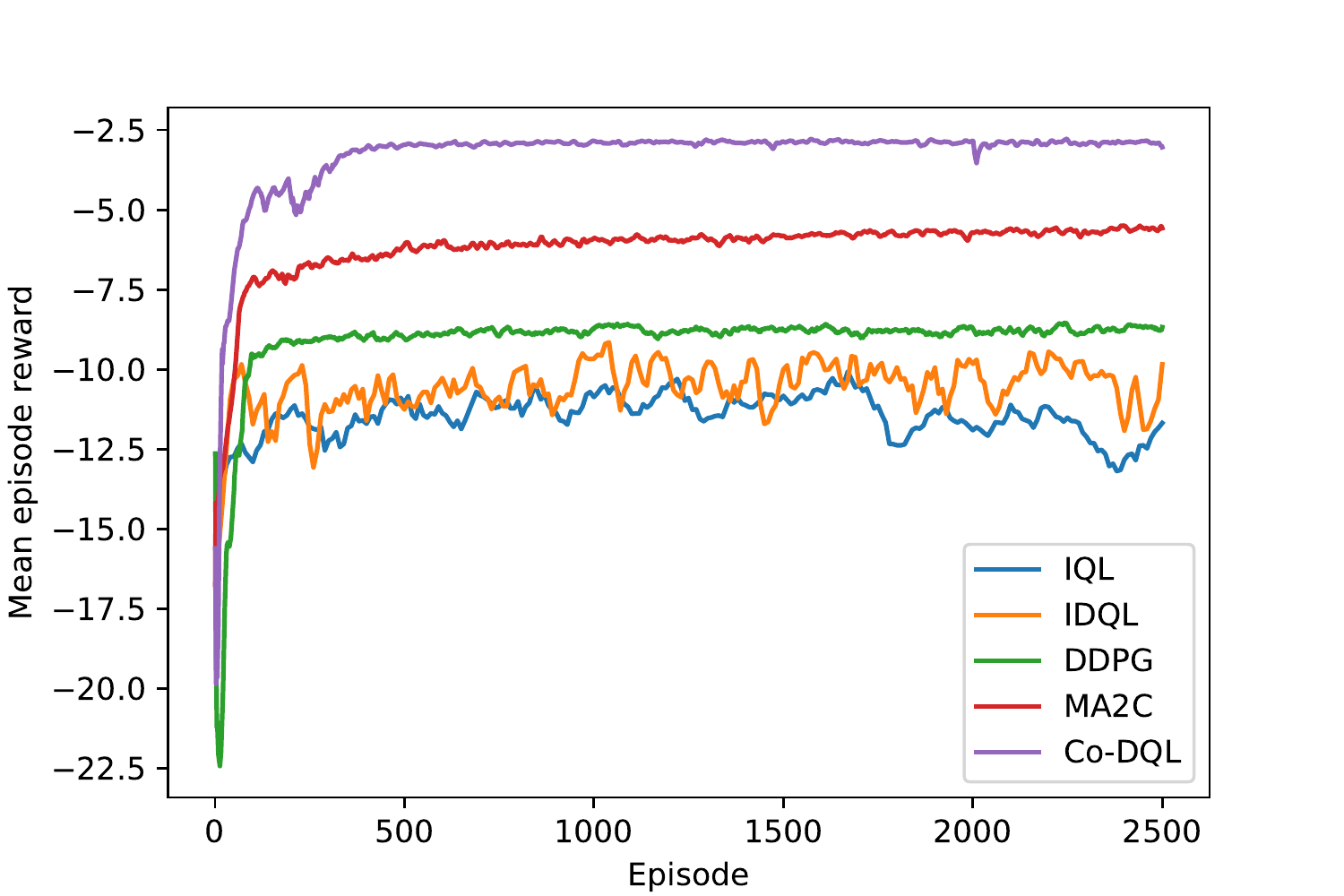}}
\caption{Reward curve of signal agent during training in the global random traffic flow scenario.\label{fig44}}
\end{figure}

As shown in the Fig.~\ref{fig33} (a), under the condition that signal agents adopt a random strategy, the mean reward reaches stable after about 2000  time steps, which means that the traffic flow of the simulator reaches a stable state too.
In order to ensure the diversity of training samples and avoid over-fitting some traffic flow states as far as possible, we record 10 discrete simulator states (i.e. vehicle position, driving status, signal status) after 2000 time steps as random seeds and it will be used to train and evaluate these methods.
In the global random traffic flow, we set the number of new vehicles added at each time step to 5, which corresponds to a high level of traffic congestion.

\emph{Result Analysis}. We run 2500 episodes for training all five models, and regularly save the trained models. The mean reward curve of signal agents is shown in Fig.~\ref{fig44}.
It can be seen from the figure that IQL suffers the lowest training performance. Although IDQL is just slightly better than IQL, the results tend to indicate that over-estimation of action-value function will damage the performance of signal control and that using double estimators can improve the performance to a certain extent.
Interestingly, the performance of DDPG is better than that of IDQL, it may be due to the advantages of actor-critic structure.
Although MA2C and Co-DQL both have more robust learning ability, Co-DQL greatly outperforms all the other methods. Co-DQL uses mean field approximation to directly model the strategies of other agents, thus it can learn a good cooperative strategies and maximize the total reward of the road network.
\begin{table}[!t]
\renewcommand{\arraystretch}{1.3}
\caption{Model Performance in Global Random Traffic Flow Scenario}
\label{tab2}
\centering
\begin{tabular}{|c|c|c|}
\hline
Method & Average Delay Time [t] & Mean Episode Reward \\
\hline
IQL  & $148.500(\pm8.963)$ & $-11.602(\pm0.700)$ \\ 
\hline
IDQL  & $131.854(\pm7.534)$ & $-10.301(\pm0.589)$ \\
\hline
DDPG  & $111.057(\pm0.606)$ & $-8.676(\pm0.047) $  \\
\hline
MA2C  & $71.553(\pm0.5812)$ & $-5.590(\pm0.045) $  \\
\hline
Co-DQL  & $36.981(\pm0.509)$ & $-2.889(\pm0.040)$  \\
\hline
\end{tabular}
\vspace{5pt}

\emph{t} means discrete time step.
\end{table}

For each algorithm, the best model obtained in the training process is used to test in this scenario.
We evaluate all of them over 100 episodes.
Table~\ref{tab2} shows the results of evaluation. Average delay time is calculated from the total delay time of vehicles in the road network during an episode. The standard deviation is given in parentheses after the mean value. Co-DQL greatly reduces the average delay time compared with the other methods. The test results are basically consistent with the trained model performance, which shows the validity of our trained model.
\subsubsection{double-ring traffic flow }\label{subsec5.2.2}
\begin{table}[!t]
\renewcommand{\arraystretch}{1.3}
\caption{Model Performance in Double-Ring Traffic Flow Scenario}
\label{tab3}
\centering
\begin{tabular}{|c|c|c|}
\hline
Method & Average Delay Time [t] & Mean Episode Reward \\
\hline
IQL  & $89.838(\pm5.645)$  & $-5.615(\pm0.353)$ \\
\hline
IDQL  & $83.921(\pm2.273)$   & $-5.245(\pm0.142)$ \\
\hline
DDPG  & $86.581(\pm1.182)$   & $-5.411 (\pm0.074)$  \\
\hline
MA2C  & $58.857(\pm0.779)$  & $-3.679(\pm0.049) $  \\
\hline
Co-DQL  & $26.046(\pm0.751)$   & $-1.628(\pm0.047)$  \\
\hline
\end{tabular}
\vspace{5pt}

\emph{t} means discrete time step.
\end{table}
\begin{figure}[!t]
\centering{\includegraphics[width=8cm,height=5cm]{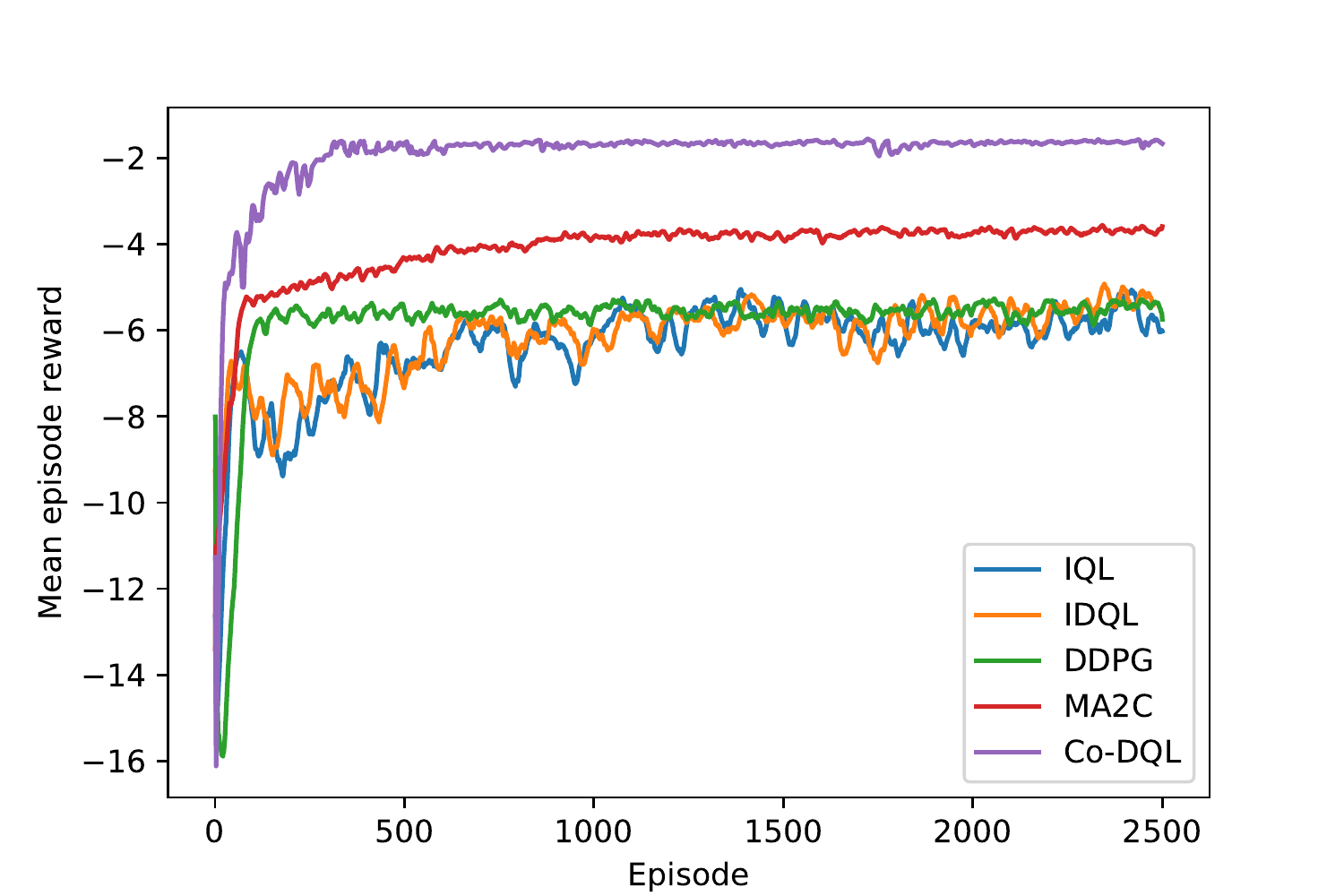}}
\caption{Reward curve of signal agent during training in the double-ring traffic flow scenario.\label{fig55}}
\end{figure}
Fig.~\ref{fig33} (b) shows the mean reward curve of agents using random strategies in double-ring traffic flow scene. Similarly, 10 simulator states are selected as seeds. In this scenario, we set the number of new vehicles added to the network at each time step to 4, which corresponds to a medium level of traffic congestion. The other parameters of the simulator are the same as those in Section~\ref{subsec5.2.1}.

\emph{Result Analysis}. Similarly, we train all the models in this scenario and save the model with the best training performance. The mean reward curve is shown in Fig.~\ref{fig55}. As expected, the training performance of Co-DQL method still outperforms all the other methods.
In addition, mainly due to the information transfer among agents, MA2C can obtain better training results in contrast to the independent agent methods, that is, IQL and IDQL.
However, although the convergence rates of DDPG, IQL and IDQL are different, the final training results are basically similar.
This may be because the problem of double-ring traffic flow is relatively simple, so these three methods can achieve relatively consistent results. In this scenario, the evaluation results are shown in Table~\ref{tab3}. Co-DQL  can obtain shorter average delay time and smaller standard deviations than other methods.

\subsubsection{four-ring traffic flow}\label{subsec5.2.3}
\begin{figure}[!t]
\centering{\includegraphics[width=8cm,height=5cm]{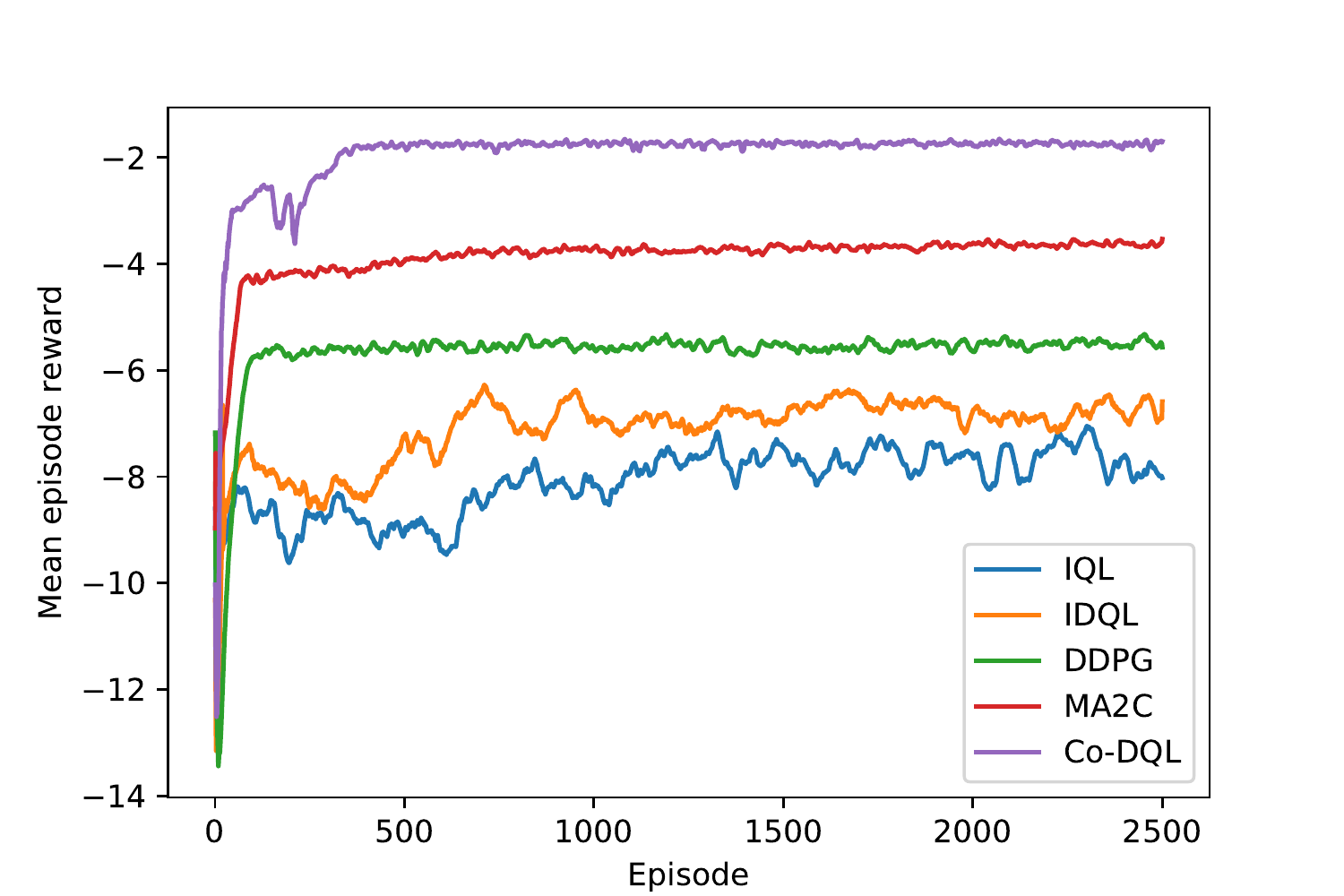}}
\caption{Reward curve of signal agent during training in the four-ring traffic flow scenario.\label{fig66}}
\end{figure}

\begin{table}[!t]
\caption{Model Performance in Four-Ring Traffic Flow Scenario}
\label{tab4} \centering
\begin{tabular}{|c|c|c|}
\hline
Method & Average Delay Time [t]  & Mean Episode Reward \\
\hline
IQL  & $168.526(\pm2.673)$  & $-7.900(\pm0.125)$ \\
\hline
IDQL  & $143.986(\pm3.761)$  & $-6.749(\pm 0.176)$ \\
\hline
DDPG  & $116.823(\pm1.610)$  & $-5.476 (\pm 0.075)$  \\
\hline
MA2C  & $77.633(\pm0.660)$  & $-3.639(\pm0.031) $  \\
\hline
Co-DQL  & $37.174(\pm0.937)$ & $-1.743 (\pm 0.044)$ \\
\hline
\end{tabular}
\vspace{5pt}

\emph{t} means discrete time step.
\end{table}
Select seeds for the four-ring traffic flow according to the curve of Fig.~\ref{fig33} (c). In order to simulate traffic conditions with low level of traffic congestion, we set the number of new vehicles added to the road network at each time step to 3. The other parameters of the simulator are set in the same way as other scenarios.

\emph{Result Analysis}. The training curve in this scenario is shown in Fig.~\ref{fig66}, and the test results are shown in Table~\ref{tab4}.
In this scenario, the training performance of IDQL is significantly better than that of IQL without double estimators.
The learning process of Co-DQL and MA2C is relatively stable and the standard deviation in the evaluation process is smaller than that of IQL, IDQL  and DDPG, this may be due to that they share information among agents. But ultimately, Co-DQL achieves the shortest average delay time by means of mean field approximation for opponent modeling and local information sharing.

\subsection{Experiment in The More Realistic TSC Simulator}\label{subsec5.3}
\begin{figure}[!t]
\centering{\includegraphics[width=8cm,height=5cm]{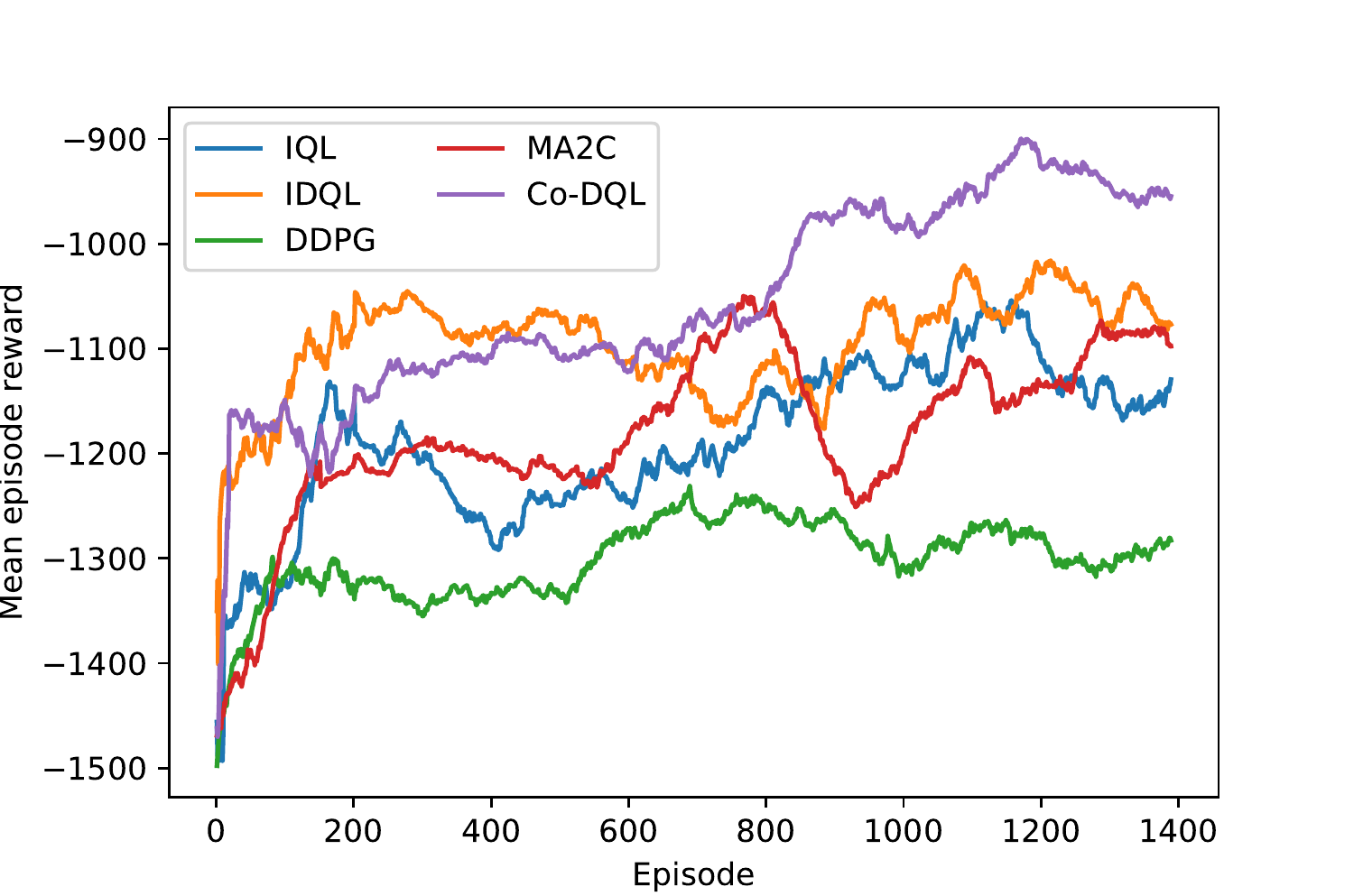}}
\caption{Reward curve of signal agent during training on the real road network with asymmetric geometry.
\label{fig999}}
\end{figure}

\begin{table*}[!t]
\caption{Model performance in real road network with asymmetric geometry}
\label{tab55} \centering
\begin{tabular}{|c|c|c|c|c|c|}
\hline
Metrics & IQL  & IDQL & DDPG & MA2C & Co-DQL\\
\hline
Mean Episode Reward  & $-1160.52(\pm190.62)$  & $-1076.34(\pm193.53)$ & $-1296.68(\pm140.87)$  & $-1108.52(\pm83.41)$ & $-930.38(\pm87.45)$\\
\hline
Avg. Vehicle Speed [m/s]  & $4.33(\pm0.49)$  & $4.53(\pm 0.45)$ & $3.81(\pm0.35)$  & $4.65(\pm0.23)$ & $5.35(\pm0.26)$\\
\hline
Avg. Intersection Delay [s/veh]  & $28.52(\pm5.55)$  & $27.17(\pm 5.42)$ & $33.01(\pm4.50)$ & $27.98(\pm2.57)$ & $20.31(\pm2.55)$ \\
\hline
Avg. Queue Length [veh]  & $10.03(\pm2.05)$  & $10.01(\pm2.12) $ & $12.53(\pm1.87)$  & $9.80(\pm1.21)$ & $7.51(\pm1.29)$ \\  
\hline
Trip Delay[s]  & $278.38(\pm35.35)$ & $254.20 (\pm 46.04)$ & $311.34(\pm30.01)$  & $253.23(\pm14.01)$ & $177.73(\pm16.70)$\\
\hline
Trip Arrived Rate  & $0.74(\pm0.08)$ & $0.80 (\pm 0.07)$ & $0.57(\pm0.05)$  & $0.79(\pm0.03)$ & $0.91(\pm0.03)$\\
\hline
\end{tabular}
\vspace{5pt}
\end{table*}
\emph{Experiment Settings}. Experiment with the simulator setup described in Section~\ref{subsec4.2.2}. Regarding MDP setting, the regularization rate  $\beta$ in reward is set to $0.2veh/s$, and the regularization factors of $wave$, $wait$, and reward are $5veh$, $100s$, and $2000veh$. Here, we train all MARL models around 1400 episodes given episode horizon $T=720$ steps, then evaluate the trained models over 10 episodes.

\emph{Result Analysis}. The mean episode reward curve during the training in this scenario is shown in Fig.~\ref{fig999}. In this challenging scenario, DDPG suffers from the worst training performance, which may be due to the time-varying traffic flow leading to a large variance of critics, so it can not effectively guide the learning of actors. Surprisingly, although the training performance of MA2C is much better than that of DDPG, it has no obvious advantage over
IQL and IDQL. This may be due to MA2C is more sensitive to the number of agents, and the setting of many hyper-parameters involved is also a big challenge. As expected, Co-DQL achieves the best training performance.

In this more realistic simulator, we have the opportunity to consider more traffic metrics than in the simplified one. Table~\ref{tab55} shows the evaluation results using ten different random seeds, in which Avg. Vehicle Speed is calculated by dividing the total distance traveled by the driving time, Avg. Intersection Delay is calculated by dividing the total delay time of each intersection by the total number of vehicles at the intersection, and Avg. Queue Length is calculated by the queue length of each time period, and Trip Delay  refers to the total delay time of vehicles in the driving process, and Trip Arrived Rate is calculated by dividing the number of vehicles that have arrived at the destination before the end of the simulation by the total number of vehicles.
 The comparison results in terms of  all measures are relatively consistent.

 According to the results, over-estimation makes a difference in the performance between IQL and IDQL, and the use of double estimators in IDQL always has a slight advantage over IQL according to most of the measurements. Compared with IQL, IDQL and DDPG, Co-DQL and MA2C show more robust test performance (less standard deviation), which shows that information sharing among agents brings benefits to cooperation among agents, and Co-DQL achieves the best average performance with respect to  multiple measures, which shows the advantage of mean field approximation in agent behavior modeling.

\subsection{Discussions}\label{subsec5.4}
\begin{figure}[!t]
\centering{\includegraphics[width=8cm,height=5cm]{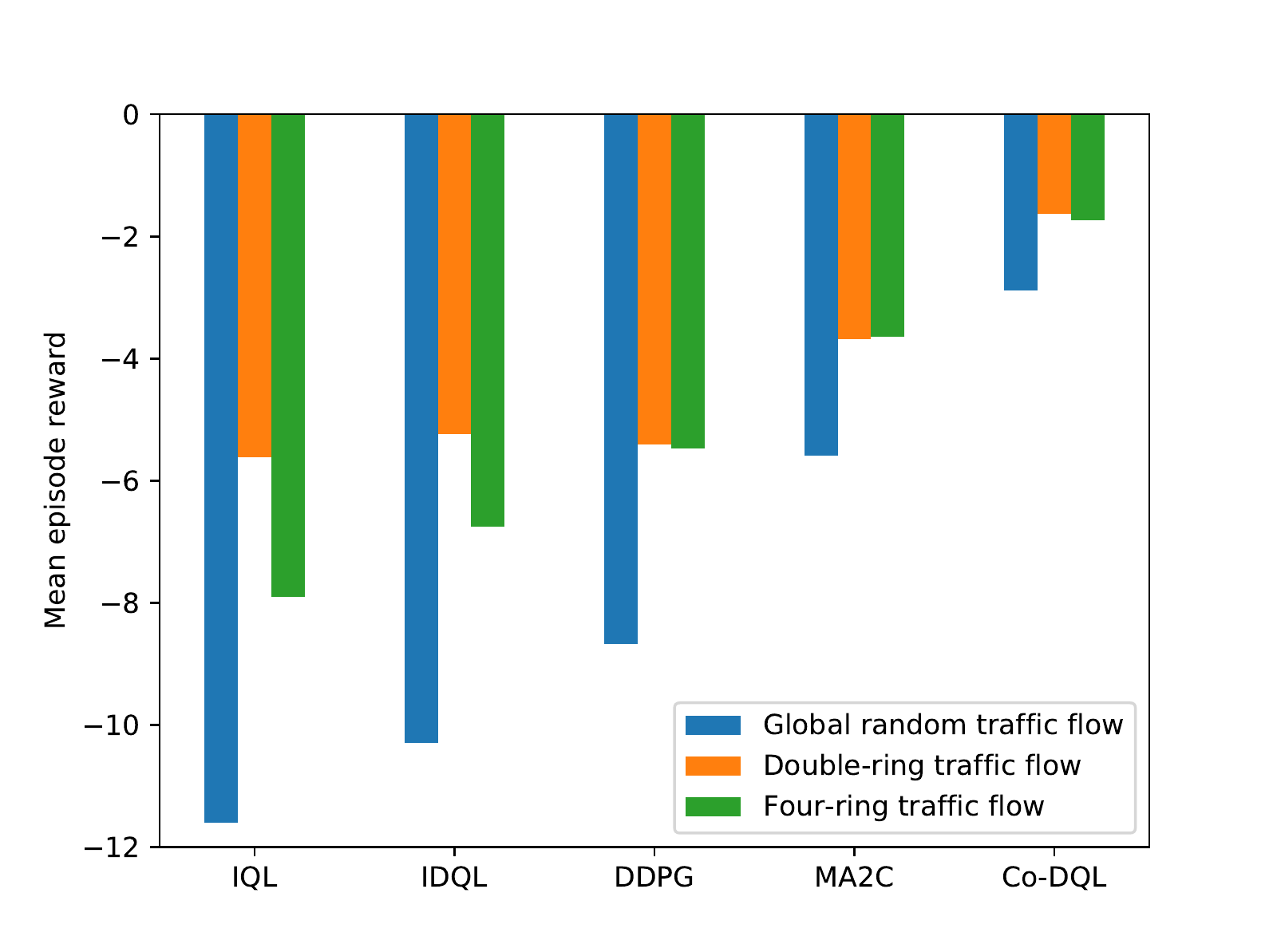}}
\caption{Mean episode reward comparison for testing the corresponding model in different traffic flow scenarios of simplified TSC. \label{fig77}}
\end{figure}
\begin{figure}[!t]
\centering{\includegraphics[width=8cm,height=5cm]{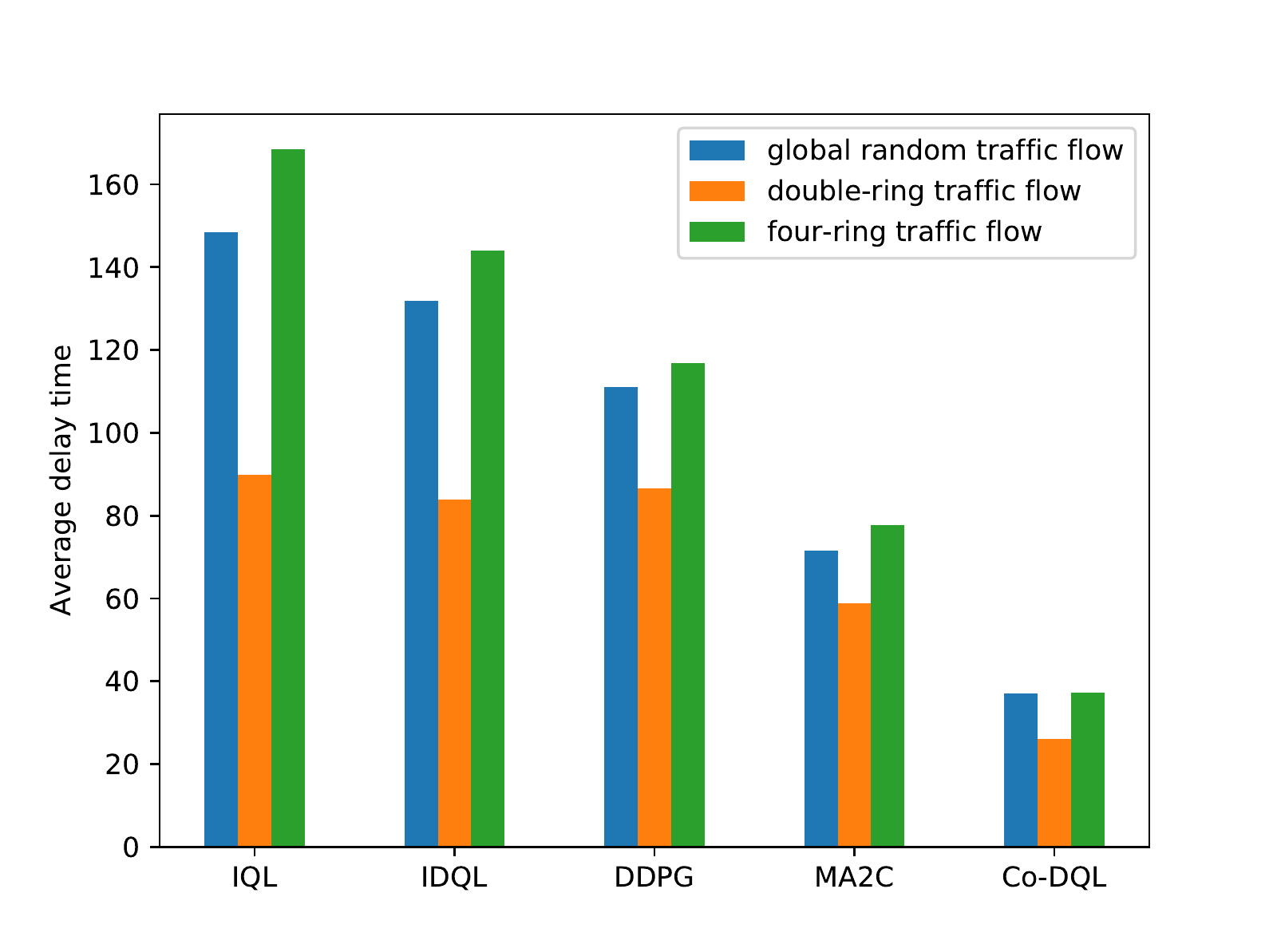}}
\caption{Mean delay time comparison for testing the corresponding model in different traffic flow scenarios of simplified TSC. \label{fig88}}
\end{figure}

Firstly, we discuss the performance of different algorithms in three traffic flow scenarios with the simplified MDP setting.
As seen from Fig.~\ref{fig77} (blue bar), all methods have a smaller mean episode reward in the global random traffic flow scenario than in the other scenarios, which is due to the highest level of traffic congestion and the largest traffic volume in this scenario.
According to Fig.~\ref{fig88} (green bar), although the mean episode reward level of each evaluation model in the four-ring traffic flow scenario is moderate, the number of vehicles in this scenario is small, which may lead to greater average vehicle delay.
Although the traffic volume of double-ring traffic flow scenario is larger than that of four-ring traffic flow scenario, the evaluation results in the former scenario (orange bar) are even slightly better than the latter (green bar), regardless of the mean episode reward of agent or the average waiting time of vehicle.
The analysis shows that the double-ring traffic flow scenario just needs the cooperation between two groups of agents, namely, the cooperation of signal agents in the inner and outer loop, while the four-ring traffic flow scenario needs the collaboration among  four groups, so the cooperation task of signal agents  in the latter may be more complex.

Experimental results on multiple scenarios show that the performance of the algorithm with double estimator is always better than that without double estimator.
Compared with the simplified situation, in the more realistic case, MA2C does not achieve the desired performance. Co-DQL can still get more training reward and better evaluation performance than the state-of-the-art decentralized MARL algorithms. In addition, we also conducted an experiment on a $7\times7$ grid road network simulator, the setting and results about the experiment are shown in the supplementary materials. One can notice that Co-DQL can achieve the best results.

In the society of RL, a hot topic is how to use it in reality. Because the uncertainty brought by the exploration behavior of RL model in the training process is a potential safety hazard for the application of TSC in practice, the training stage of our model is completed in a TSC simulator in a similar way as most RL models \cite{wei2018intellilight}\cite{casas2017deep}\cite{chu2019multi}, and the model deployed in reality is generally the model trained in the simulator. Although there is a gap between the simulator and the real environment, simulation to reality (sim2real)\cite{sadeghi2018sim2real}, as a branch of RL, has been widely studied in order to bridge the gap.

\section{Conclusion}\label{sec6}
When to design a MARL algorithm, a critical challenge is how to make the agents efficiently cooperate, and one of the breach of realize is properly estimating the Q values and sharing local information among agents. Along this line of thought, this paper developed Co-DQL, which takes advantage of some important ideas studied in the literature. In more detail, Co-DQL employs an independent double Q-learning method based on double estimators and the UCB exploration, which can eliminate the over-estimation of traditional independent Q-learning while ensuring exploration. It adopts mean field approximation to model the interaction among agents so that agents can learn a better cooperative strategy.
In addition, we presented a reward allocation mechanism and a local state sharing method. Based on the characteristics of TSC, we gave the details of the algorithmic elements. To validate the performance of the proposed algorithm, we tested Co-DQL on various
traffic flow scenarios of TSC simulators. Compared with several state-of-the-art MARL algorithms (i.e., IQL,  IDQL, DDPG and MA2C), Co-DQL can achieve promising results.

In the future, we hope to further test Co-DQL on the real city road network, and  we will consider other approaches on large-scale MARL such as  hierarchical architecture \cite{tan2019cooperative} \cite{vezhnevets2017feudal}.
In addition, note that the local optimization of an agent's reward (throughput) may reduce the neighboring agents' rewards in a nonlinear way. Such a nonlinearity is typical in traffic flow. Using the linear weighted function with a constant $\alpha$ may not fully capture the nonlinear throughput relationship between neighboring intersections. Also, each agent's reward will appear multiple times, depending on the number of connected neighboring intersections. For instance, an intersection with five legs will receive more weights than a three-leg
intersection that may cause a biased optimal solution.
Hence, it may be interesting to further study on the reward allocation mechanism.

So far, a great number of methods have been proposed for TSC, such as max pressure\cite{varaiya2013max}, cell transmission model\cite{timotheou2014distributed}. It may be interesting to comprehensively compare these methods. Furthermore, parameters heavily affect the performance of an algorithm, it is interesting to study how to automatically adjust them so as to achieve the promising quality. Finally, it may be interesting to study our method on the other MDP settings for TSC problem.

\section*{Acknowledgment}
 This work was supported by the National Natural Science Foundation of China (No. 61973244, 61573277).

\ifCLASSOPTIONcaptionsoff
  \newpage
\fi


\bibliographystyle{IEEEtran}
\bibliography{mybibfile}

%
%


\end{document}